\documentclass{article}

\usepackage{url}
\usepackage{color,graphicx,amssymb,xspace,setspace, amsmath}
\usepackage[small,bf]{caption} 
\definecolor{orange}{rgb}{.6,0.1,.6}

\usepackage[inline=true,margin=false]{fixme}
\usepackage{tabularx}
\usepackage{colortbl}
\usepackage{rotating}
\usepackage{times}
\usepackage{textcomp}
\usepackage{subfig}
\usepackage[algo2e,linesnumbered,ruled]{algorithm2e} 

\usepackage{appendix}
\usepackage{float}

\usepackage{mdwtab}

\usepackage{hyperref}

\newcommand{\bA}{\boldsymbol{A}}

\newcommand{\bB}{\boldsymbol{B}}
\newcommand{\bQ}{\boldsymbol{Q}}

\newcommand{\bo}{\boldsymbol{0}}
\newcommand{\bI}{\boldsymbol{I}}
\newcommand{\bX}{\boldsymbol{X}}
\newcommand{\bS}{\boldsymbol{S}}

\newcommand{\bU}{\boldsymbol{U}}
\newcommand{\bV}{\boldsymbol{V}}

\newcommand{\bSig}{\boldsymbol{\Sigma}}
\newcommand{\R}{\mathbb{R}}
\newcommand{\bC}{\boldsymbol{C}}

\newcommand{\bR}{\boldsymbol{R}}
\newcommand{\bW}{\boldsymbol{W}}
\newcommand{\hX}{\widehat{\boldsymbol{A}}}

\newcommand{\bE}{\boldsymbol{E}}
\newcommand{\E}{\mathbb{E}}

\newtheorem{theorem}{Theorem}

\newtheorem{lemma}{Lemma}

\newtheorem{definition}{Definition}
\newtheorem{corollary}{Corollary}

\newtheorem{proof}{Proof}

\begin{document}

\title{Block CUR: Decomposing Matrices \\using Groups of Columns}

\author{Urvashi~Oswal,
	Swayambhoo~Jain,
        Kevin~S.~Xu,
        and~Brian~Eriksson
        \thanks{This research was performed while U.O., K.S.X., and B.E. were at Technicolor. UO is with the Department of Electrical and Computer Engineering, University of Wisconsin -- Madison, SJ is with Technicolor Research -- Palo Alto, KSX is with the Electrical Engineering and Computer Science Department, University of Toledo, BE is with Adobe -- San Jose. Author emails: {\tt uoswal@wisc.edu, swayambhoo.jain@technicolor.com, kevin.xu@utoledo.edu, brian.c.eriksson@gmail.com }}}
        
\date{}

\maketitle

\begin{abstract}
A common problem in large-scale data analysis is to approximate a matrix using a combination of specifically sampled rows and columns, known as CUR decomposition. Unfortunately, in many real-world environments, the ability to sample specific individual rows or columns of the matrix is limited by either system constraints or cost. In this paper, we consider matrix approximation by sampling predefined \emph{blocks} of columns (or rows) from the matrix. We present an algorithm for sampling useful column blocks and provide novel guarantees for the quality of the approximation. This algorithm has application in problems as diverse as biometric data analysis to distributed computing. We demonstrate the effectiveness of the proposed algorithms for computing the Block CUR decomposition of large matrices in a distributed setting with multiple nodes in a compute cluster, where such blocks correspond to columns (or rows) of the matrix stored on the same node, which can be retrieved with much less overhead than retrieving individual columns stored across different nodes. In the biometric setting, the rows correspond to different users and columns correspond to users' biometric reaction to external stimuli, {\em e.g.,}~watching video content, at a particular time instant. There is significant cost in acquiring each user's reaction to lengthy content so we sample a few important scenes to approximate the biometric response. An individual time sample in this use case cannot be queried in isolation due to the lack of context that caused that biometric reaction. Instead, collections of time segments ({\em i.e.,} blocks) must be presented to the user. The practical application of these algorithms is shown via experimental results using real-world user biometric data from a content testing environment.
\end{abstract}

\section{Introduction}
\label{sec:intro}

The ability to perform large-scale data analysis is often limited by two opposing forces.  The first force is the need to store data in a matrix format for the purpose of analysis techniques such as regression or classification.  The second force is the inability to store the data matrix completely in memory due to the size of the matrix in many application settings.  This conflict gives rise to storing factorized matrix forms, such as SVD or CUR decompositions~\cite{drineas08}.  

We consider a matrix $\bA$ with $m$ rows and $n$ columns, {\em i.e., }$\bA \in \R^{m \times n}$.  Using a truncated $k$ number of singular vectors ({\em e.g.,} where $k < \min\left\{m,n\right\}$), the singular value decomposition (SVD) provides the best rank-$k$ approximation to the original matrix.  The singular vectors often do not preserve the structure in original data. Preserving the original structure in the data may be desirable due to many reasons including interpret-ability in case of biometric data or for storage efficiency in case of sparse matrices. This has led to the introduction of the CUR decomposition, where the factorization is performed with respect to a subset of rows and columns of the matrix itself.  This specific decomposition describes the matrix $\bA$ as the product of a subset of matrix rows $\bR$ and a subset of matrix columns $\bC$ (along with a matrix $\bU$ that fits $\bA \approx \bC\bU\bR$).

Significant prior work has examined how to efficiently choose the rows and columns in the CUR decomposition and has derived worst-case error bounds ({\em e.g.,}~\cite{drineas09}). These methods have been applied successfully to many real-world problems including genetics \cite{app3}, astronomy \cite{app2}, and mass spectrometry imaging \cite{app1}. Unfortunately, a primary assumption of current CUR techniques, that individual rows and columns of the matrix can be queried, is either impossible or quite costly in many real world problems and instead require a block approach. 

In this paper, we consider the following two applications which represent the two main motivating factors for considering block decompositions. 

\textbf{Biometric data analysis.}  In applications where the ordering of rows or columns is meaningful, such as images, video, or speech data matrices, sampling contiguous blocks of columns adds contextual information that is necessary for interpretability of the factorized representation. One emerging application is audience reaction analysis of video content using biometrics. We focus on the scenario where users watch video content while wearing sensors, and changes in biometric sensors indicate changes in reaction to the content. For example, increases in heart rate or a spike in electrodermal activity indicate an increase in content engagement. In this paper, a matrix of biometric data such as \textit{Electrodermal Activity} (EDA) is collected from users reacting to external stimuli, e.g., watching video content. In prior work, EDA has shown to be useful for a variety of user analytics tasks to assess the reaction of viewers\cite{silveira2013predicting,jain2017compressed}. In this setting, $m$ is the number of users and $n$ corresponds to the number of time samples for which biometric reaction is collected. Unfortunately, there is significant cost in acquiring each user's reaction to lengthy content so instead we collect full responses (corresponding to some rows of the matrix) from only a limited number of users. For remaining users, we propose to collect responses for only a few important scenes of the video (corresponding to column blocks of the matrix) as shown in Figure~\ref{fig:biometric} and then \textit{approximate} their full response. An individual time sample in this use case cannot be queried in isolation due to the lack of context that caused that biometric reaction. Instead, collections of time segments (i.e., blocks) must be presented to the user. In this setting block sampling can be viewed as a \textbf{restriction} which leads to more interpretable solutions.

\textbf{Distributed storage systems.} Large-scale datasets often require distributed storage, a regime where there can be substantial overhead involved in querying individual rows or columns of a matrix.  In these regimes, it is more efficient to retrieve predefined \emph{blocks} of rows or columns at one time corresponding to the rows or columns stored on the same node, as shown in Figure \ref{fig:distributed}, in order to minimize the overhead in terms of latency while keeping the throughput constant. In doing so, one forms a Block CUR decomposition, with more details provided in Section~\ref{sec:experiments}. 
Current CUR decomposition techniques do not take advantage of this predefined block structure.

\begin{figure}[t]
\centering
	\subfloat[]{
	  \label{fig:biometric}
	  \includegraphics[width=0.46\textwidth]{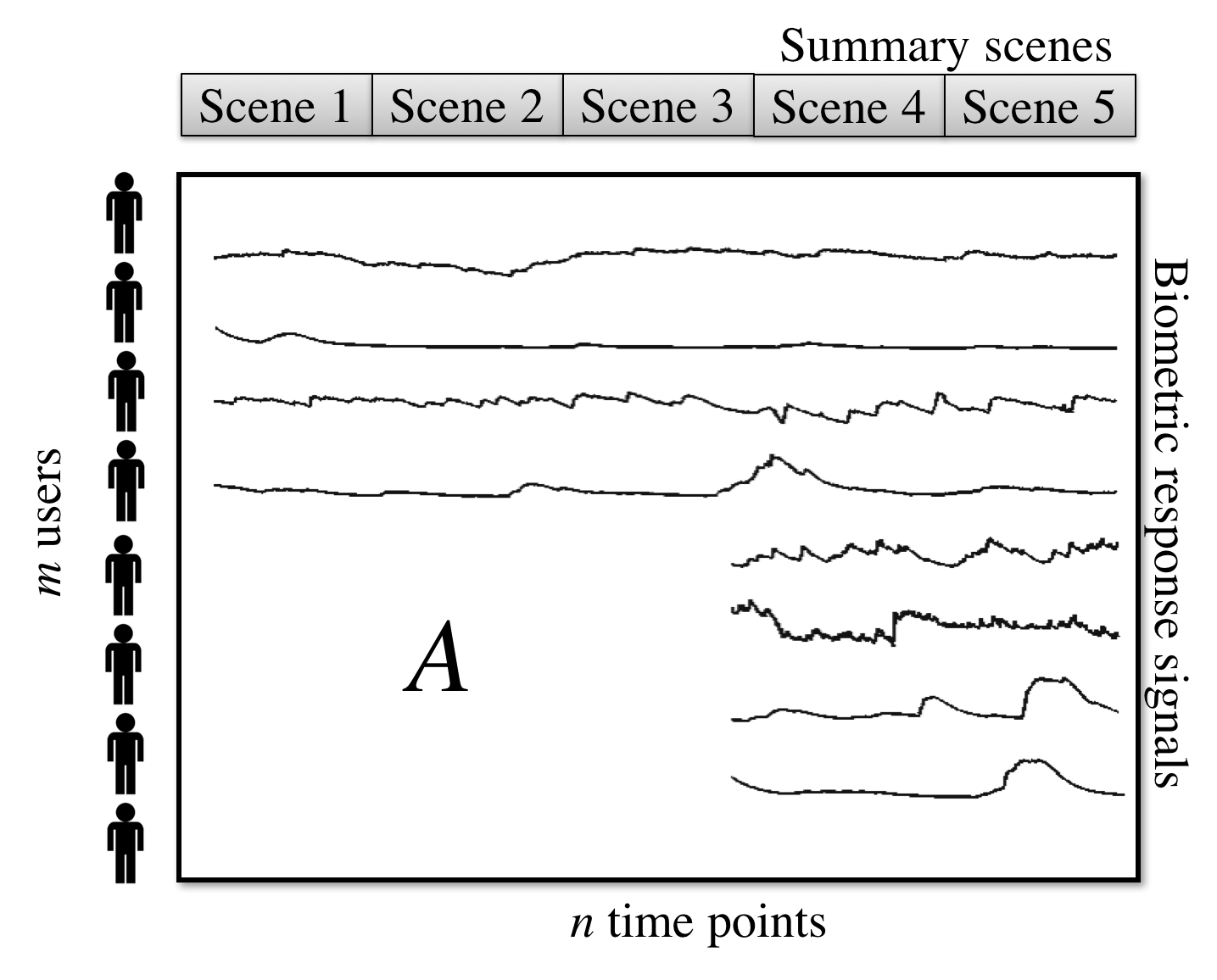}
	}
	\quad
	\subfloat[]{
	  \label{fig:distributed}
	  \includegraphics[width=0.46\textwidth]{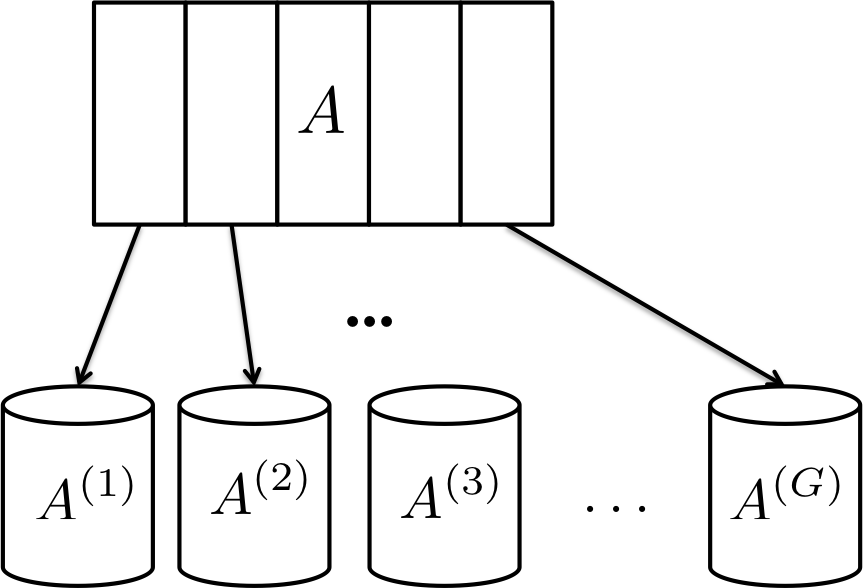} 
	}
 
\caption[]{Applications: \subref{fig:biometric} Biometric data analysis. Blocks of columns or time instances correspond to scenes in a video and provide context for biometric reaction. \subref{fig:distributed} Distributed storage of a large matrix across multiple nodes in a cluster. Blocks are allocated to each of the $G$ nodes. }
 \label{fig:blockCURapp} 
\end{figure}

\textbf{Main contributions.} Using these insights into real-world applications of CUR decomposition, this paper makes a series of contributions.  We propose a simple randomized Block CUR algorithm for subset selection of rows and blocks of columns and derive novel worst-case error bounds for this randomized algorithm. On the theory side, we present new theoretical results related to approximating matrix multiplication and generalized $\ell_2$ regression in the block setting. These results are the fundamental building blocks used to derive the error bounds for the presented randomized algorithms.  The sample complexity bounds feature a non-trivial dependence on the matrix partition, $i.e.,$ the distribution of information in the blocks of the matrix. This dependence is non-trivial in that it cannot be obtained by simply extending the analysis of the original individual column CUR setting to the Block CUR setting. As a result, our analysis finds a sample complexity improvement on the order of the \textit{block stable rank} of a matrix (See Table \ref{Tab:Complexity} in Section \ref{sec:algorithm}).

On the practical side, this algorithm performs fast block sampling taking advantage of the natural storage of matrices in distributed environments (See Table \ref{Tab:Latency} in Section \ref{sec:experiments}). We demonstrate empirically that the proposed Block CUR algorithms can achieve a significant speed-up when used to decompose large matrices in a distributed data setting.  We conduct a series of CUR decomposition experiments using Apache Spark on Amazon Elastic Map-Reduce (Amazon EMR) using both synthetic and real-world data.  In this distributed environment, we find that our Block CUR approach achieves a speed-up of $2$x to $6$x for matrices larger than $12000\times 12000$.  This is compared with previous CUR approaches that sample individual rows and columns and while achieving the same matrix approximation error rate. We also perform experiments with real-world user biometric data from a content testing environment and present interesting use cases where our algorithms can be applied to user analytics tasks.

\section{Setup and background}
\label{sec:setup}
\subsection{Notation}

Let $\bI_k$ denote the $k \times k$ identity matrix and $\bo$ denote a zero matrix of appropriate size. We denote vectors (matrices) with lowercase (uppercase) bold symbols like $\boldsymbol{a}$  ($\bA$). The $i$-th row (column) of a matrix is denoted by $\bA_i$ ($\bA^i$). We represent the $i$-th block of rows of a matrix by $\bA_{(i)}$  and the $i$-th block of columns of a matrix by $\bA^{(i)}$. 

Let $[n]$ denote the set $\{1, 2,\dots, n\}$. Let $\rho =\text{rank}(\bA) \leq \min\{m,n\}$ and $k \leq \rho$. The singular value decomposition (SVD) of $\bA$ can be written as $\bA = \bU_{A,\rho} \bSig_{A,\rho} \bV_{A,\rho}^T$ where $\bU_{A,\rho} \in \R^{m \times \rho}$ contains the $\rho$ left singular vectors; $\bSig_{A,\rho} \in \R^{\rho \times \rho}$ is the diagonal matrix of singular values, $\sigma_i(\bA)$ for $i = 1, \dots, \rho$; and $\bV_{A,\rho}^T \in \R^{\rho \times n}$ is an orthonormal matrix containing the $\rho$ right singular vectors of $\bA$. We denote $\bA_k = \bU_{A,k} \bSig_{A,k} \bV_{A,k}^T$ as the best rank-$k$ approximation to $\bA$ in terms of Frobenius norm. The pseudoinverse of $\bA$ is defined as $\bA^{\dagger} = \bV_{A,\rho} \bSig_{A,\rho}^{-1} \bU_{A,\rho}^T$. Also, note that $\bC\bC^{\dagger}\bA = \bU_C\bU_C^T\bA$ is the projection of $\bA$ onto the column space of $\bC$, and  $\bA\bR^{\dagger}\bR = \bA\bV_{R,k}\bV_{R,k}^T$ is the projection of $\bA$ onto the row space of $\bR$. 

The Frobenius norm and spectral norm of a matrix are denoted by $\|\bA\|_F$ and $\|\bA\|_2$ respectively. The square of the Frobenius norm is given by $\| \bA\|_F^2 =$ $  \sum_{i = 1}^{m}\sum_{j = 1}^{n} A_{i,j}^2 $  $= \sum_{i = 1}^{k} \sigma_i^2(\bA)  $. The spectral norm is given by  $\| \bA \|_2 =  \max_i \sigma_i(\bA)$.

\subsection{The CUR problem and other related work}

The need to factorize a matrix using a collection of rows and columns of that matrix has motivated the CUR decomposition literature. CUR decomposition is focused on sampling rows and columns of the matrix to provide a factorization that is close to the best rank-$k$ approximation of the matrix. One of the most fundamental results for a CUR decomposition of a given matrix $\bA \in \R^{m \times n}$ was obtained in \cite{drineas08}. We re-state it here for the sake of completion and setting the appropriate context for our results to be stated in the next section. This relative error bound result is summarized in the following theorem.

\begin{theorem}\label{ThmCUR} (Theorem 2 from~\cite{drineas08} applied to $\bA^T$)
Given $\bA \in \R^{m \times n}$ and an integer $k \le \min \{m,n\}$, let $r = O(\frac{k^2}{\varepsilon^{2}} \ln(\frac{1}{\delta}))$ and $c = O(\frac{r^2}{\varepsilon^{2}} \ln(\frac{1}{\delta}) )$. 
There exist randomized algorithms such that, if $c$ columns are chosen to construct $\bC$ and $r$ rows are chosen to construct $\bR$, then with probability $\geq 1 - \delta$, the following holds:
 \begin{equation*} \| \bA - \bC \bU \bR\|_F \leq (1 + \varepsilon ) \| \bA -  \bA_k \|_F \end{equation*}
  where $\varepsilon, \delta \in (0,1)$, $\bU = \bW^{\dagger}$ and $\bW$ is the scaled intersection of $\bC$ and $\bR$.
\end{theorem}

This theorem states that as long as enough rows and columns of the matrix are acquired ($r$ and $c$, respectively), then the CUR decomposition will be within a constant factor of the error associated with the best rank-$k$ approximation of that matrix. Central to the proposed randomized algorithm was the concept of sampling columns of the matrix based on a \emph{leverage score}.  The leverage score measures the contribution of each column to the approximation of $\bA$. 

\begin{definition} \label{leverage}
The {\textbf{leverage score}} of a column is defined as the squared row norm of the top-$k$ right singular vectors of $\bA$ corresponding to the column:
$$ \ell_j =  \| \bV_{A,k}^T \boldsymbol{e}_j \|_2^2,   \  j \in [n],$$ where $\bV_{A,k}$ consists of the top-$k$ right singular vectors of $\bA$ as its rows, and $\boldsymbol{e}_j $ is the $j$-th column of identity matrix which picks the $j$-th column of  $\bV_{A,k}^T$. 
\end{definition}

The CUR algorithm involves randomly sampling $r$ rows using probabilities generated by the calculated leverage scores to obtain the matrix $\bR$, and thereafter sampling $c$ columns of $\bA$ based on leverage scores of the $\bR$ matrix to obtain $\bC$. The key technical insight in \cite{drineas08} is that the leverage score of a column measures ``how much'' of the column lies in the subspace spanned by the top-$k$ left singular vectors of $\bA$; therefore, this method of samping is also known as \textit{subspace} sampling. By sampling columns that lie in this subspace more often, we get a relative-error low rank approximation of the matrix. The concept of sampling the important columns of a matrix based on the notion of subspace sampling first appeared in context of fast $\ell_2$ regression in \cite{drineas2006sampling} and was refined in \cite{drineas08} to obtain performance error guarantees for CUR matrix decomposition.

 These guarantees were subsequently improved in follow-up work \cite{drineas09}. Modified versions of this problem have been studied extensively for adaptive sampling \cite{wang12},  divide-and-conquer algorithms for parallel computations \cite{mackey11}, and input-sparsity algorithms \cite{woodruff14}. The authors of \cite{wang12} propose an adaptive sampling-based algorithm which requires only $c = O(k/\varepsilon)$ columns to be sampled when the entire matrix is known and its SVD can be computed. The authors of \cite{woodruff14} also proposed an optimal, deterministic CUR algorithm. In \cite{boutsidis2014}, the authors prove the lower bound of the column selection problem; at least $c = k/\varepsilon$ columns are selected to achieve the $(1 + \varepsilon)$ ratio. 

These prior results require sampling of arbitrary rows and columns of the matrix $\bA$ which may be either unrealistic or inefficient in many practical applications. In this paper, we focus on the problem of efficiently sampling pre-defined blocks of columns (or rows) of the matrix to provide a factorization that is close to the best rank-$k$ approximation of the matrix in the more natural environment of block sampling for biometric and distributed computation, explore the performance advantages of block sampling over individual column sampling, and provide the first non-trivial theoretical error guarantees for Block CUR decomposition. In the following section, we propose and analyze a randomized algorithm for sampling blocks of the matrix based on \textit{block leverage scores}.

\section{The Block CUR algorithm}
\label{sec:algorithm}
A block may be defined as a collection of $s$ columns or rows. For clarity of exposition, without loss of generality, we consider column blocks but the techniques and derivations also hold for row blocks by applying them to the transpose of the matrix. For ease of exposition, we also assume equal-sized blocks but one could easily extend the methods to blocks of varying sizes.  Let $G = \lceil n/s \rceil$ be the number of possible blocks in $\bA$. We consider the blocks to be predefined due to natural constraints or cost, such as data partitioning in a distributed compute cluster. 

The goal of the Block CUR algorithm is to approximate the underlying matrix $\bA$ using $g$ blocks of columns and $r$ rows, as represented in Figure~\ref{fig:blockCUR}. For example, in the biometric analysis setting each block could correspond to user reactions at a collection of time points corresponding to a scene in a movie. The goal is to approximate the users' reactions to the full movie using only their response to a summary of the movie (containing a subset of the scenes).  
\begin{figure}[t]
\centering
  \includegraphics[width=0.65\textwidth]{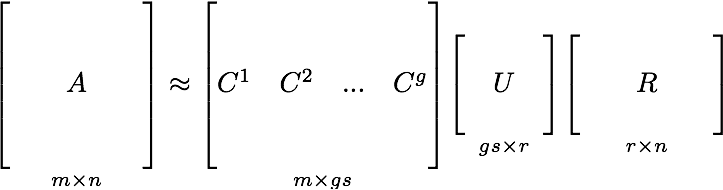}  
\caption{Example Block CUR decomposition,  where $\bC^{t} \in \R{^{m \times s}}$ for $t \in [g]$ is sampled from \{$\bA^{(j_t)}$ : $j_t \in [G]$\}. }
 \label{fig:blockCUR}
\end{figure}

Given the new regime of submatrix blocks, we begin by defining a {\em{block leverage score}} for each block of columns.
\begin{definition} \label{group} 
The {\textbf{block leverage score}} of a group of columns is defined as the sum of the squared row norms of the top-$k$ right singular vectors of $\bA$ corresponding to the columns in the block:
$$ \ell_g(\bA, k) =  \| \bV_{A,k}^T \bE_g \|_F^2,   \  g \in [G],$$ where $\bV_{A,k}$ consists of the top-$k$ right singular vectors of $\bA$, and $\bE_g $ consists of the corresponding block of columns in the identity matrix which picks the columns of $ \bV_{A,k}^T $ corresponding to the elements in block $g$. 
\end{definition}
Much like the individual column leverage scores defined in \cite{drineas08}, the block leverage scores measure how much a particular column block contributes to the approximation of the matrix $\bA$.

\subsection{Algorithm details}

The Block CUR Algorithm, detailed in Algorithm~\ref{groupalgo}, takes as input the matrix $\bA$ and returns as output an $r \times n $ matrix $\bR$ consisting of a small number of rows of $\bA$ and an $m \times c$ matrix $\bC$ consisting of a small number of column blocks from $\bA$.

\begin{algorithm2e}[h]
\caption{Block CUR}
\SetKwInOut{Input}{Input}
    \SetKwInOut{Output}{Output}
    \Input{$\bA$, target rank $k$, size of each block $s$, error parameter $\varepsilon$, positive integers $r$, $g$}
    \Output{$\bC, \bR, \hX = \bC \bU \bR$}
\begin{enumerate}
\item \textit{Row subset selection: } Sample $r$ rows uniformly from $\bA$ according to\\ $p_i  = 1/m$ for $i \in [m]$ and compute $\bR = \bS^T_R\bA$.
\item \textit{Column block subset selection: }  For $t \in [g]$, select a block of columns\\ $j_t \in [G]$ independently with probability $p_{j_t} $ $ = \frac{\ell_i(\bR, r)} {r} $$ = \frac{\| \bV_{R, r}^T \bE_i \|_F^2} {r} $\\ for $i \in [G]$ and update $\bS$, where $\bV_{R,r}$ consists of the top-$r$ right singular\\ vectors of $\bR$, and $\bE_i $ picks the columns $ \bV_{R,r}^T $ corresponding to the \\elements in block $i$.  Compute $\bC = \bA\bS$.  
\item \textit{CUR approximation}: $\hX = \bC \bU \bR$ where $\bU = \bW^{\dagger}$, and $\bW = \bR\bS$ is the \\scaled intersection of $\bR$ and $\bC$.
\end{enumerate}
\label{groupalgo}
\end{algorithm2e}

In Algorithm~\ref{groupalgo}, for $t \in [g]$, block $j_t \in [G]$ is sampled with some probability $p_{j_t}$ and scaled using matrix $\bS \in \R^{n \times gs}$. The $(j_t,t)$-th non-zero $s \times s$ block of $\bS$ is defined as $ \bS_{j_t, t} = \bI_s/\sqrt{g p_{j_t}} $ where $g = c/s$ is the number of blocks picked by the algorithm. This sampling matrix picks the blocks of columns and scales each block to compute $\bC = \bA\bS$. A similar sampling and scaling matrix $\bS_R$ is defined to pick the blocks of rows and scale each block to compute $\bR = \bS_R^T\bA$. An example of sampling matrix $\bS$ with blocks chosen in order $[1,3,2]$ is as follows:

\[
\small
\bS_{n \times gs} =  
\begin{bmatrix}
\frac{1}{\sqrt{g p_{1}}} \bI_s&\bo&\bo\\
\bo&\bo&\frac{1}{\sqrt{g p_{2}}} \bI_s\\
\bo& \frac{1}{\sqrt{g p_{3}}}\bI_s&\bo\\
 \bo&\bo&\bo
 \end{bmatrix}.
\]

In addition to considering block sampling of columns, another advantage of this algorithm is not requiring the computation of a full SVD of $\bA$. In many large-scale applications, it may not feasible to compute the SVD of the entire matrix $\bA$.  In these cases, algorithms requiring knowledge of the leverage scores cannot be used. Instead, we use an estimate of the block leverage scores called the \textit{approximate block leverage scores}. A subset of the rows (corresponding to users) are chosen uniformly at random, and the block scores are calculated using the top-$k$ right singular vectors of this row matrix instead of the entire $\bA$ matrix. This step is not the focus of the experiments in this paper so it can also be replaced with other fast approximate calculations of leverage scores involving sketching or additional sampling \cite{approxLevScores,xu14}. The advantage of using our approximate leverage scores is that the same set of rows is used to approximate the scores and also to compute the CUR approximation. Hence no additional sampling or sketching steps are required. In terms of the biometric application, each row corresponds to a user's biometric reaction to a movie. Since collecting user reactions to lengthy content can be expensive, eliminating redundant sampling leads to huge savings in resources.

The running time of Algorithm \ref{groupalgo} is essentially driven by the time required to compute the SVD of $\bR$, \textit{i.e.,}  $\mathcal{O}(SVD(\bR)$) time, and the time to construct $\bR$, $\bC$ and $\bU$. Construction of $\bR$ requires $\mathcal{O}(rn$) time, construction of $\bC$ takes $\mathcal{O}(mc$) time, construction of $\bW$ requires $\mathcal{O}(rc$) time and construction of $\bU$ takes $\mathcal{O}(r^2c$) time.

\subsection{Theoretical results and discussion}

The main technical contribution of the paper is a novel relative-error bound on the quality of approximation using blocks of columns or rows to approximate a matrix $\bA \in \R^{m \times n}$. Before stating the main result, we define two important quantities that measure important properties of the matrix $\bA$ that are fundamental to the quality of approximation. We first define a property of matrix rank relative to the collection of matrix blocks.  Specifically, we focus on the concept of {\em{matrix stable rank}} from \cite{stablerank} and define the {\em{block stable rank}} as the minimum stable rank across all matrix blocks.

\begin{definition}  
  Let $\bV_{A,k}$ consist of the top-$k$ right singular vectors of $\bA$. Then the \textbf{block stable rank} is defined as
  $$\alpha_A =  \min_{g \in [G]} \frac{\| \bV_{A, k}^T \bE_g\|_F^2}{\| \bV_{A,k}^T \bE_g\|_2^2},$$
  where $\bE_g $ consists of the corresponding block of columns in the identity matrix that picks the columns of $ \bV_{A,k}^T $ corresponding to the elements in block $g$. 
\end{definition}

Intuitively, the above definition gives a measure of how informative the worst matrix column block is. The second property is a notion of {\em{column space incoherence}}. When we sample rows uniformly at random, we can give relative error approximation guarantees when the matrix $\bA$ satisfies an incoherence condition. This avoids pathological constructions of rows of $\bA$ that cannot be sampled at random.

\begin{definition}
\label{def:incoherence}
The top-$k$ {\textbf{column space incoherence}} is defined as $$ \mu := \mu(\bU_{A,k}^T) = \frac{m}{k} \max_i \| \bU_{A,k}^T \boldsymbol{e}_i \|_2^2, $$
where $\boldsymbol{e}_i$ picks the $i$-th column of $\bU_{A,k}^T$.  
\end{definition}

The column space incoherence is used to provide a guarantee for fast approximation without computing the SVD of the entire matrix $\bA$. Equipped with these definitions, we state the main result that provides a relative-error guarantee for the Block CUR approximation in Theorem \ref{Thm2}.

\begin{theorem}\label{Thm2}
Given $\bA \in \R^{m \times n}$ with incoherent top-$k$ column space, i.e.,~$\mu \leq \mu_0$, let $r = O\left(\mu_0 \frac{k^2}{\varepsilon^{2}} \ln(\frac{1}{\delta})\right)$ and $g = O\left(\frac{ r^2}{ \alpha_R \varepsilon^{2}} \ln(\frac{1}{\delta}) \right)$. There exist randomized algorithms such that, if $r$ rows and $g$ column blocks are chosen to construct $\bR$ and $\bC$, respectively, then with probability $\geq 1 - \delta$, the following holds:
 \begin{equation*} \| \bA - \bC \bU \bR\|_F \leq (1 + \varepsilon ) \| \bA -  \bA_k \|_F, \end{equation*}
  where $\varepsilon, \delta \in (0,1)$ and $\bU = \bW^{\dagger}$ is the pseudoinverse of scaled intersection of $\bC$ and $\bR$.
\end{theorem}

We provide a sketch of the proof and highlight the main technical challenges in proving the claim in Section~\ref{sec:ProofSketch} and defer the proof details to the Appendix. In Section \ref{sec:ProofSketch}, we first provide a relative-error guarantee (Lemma \ref{Thm3}) for the approximation provided by Algorithm \ref{groupalgo}. After applying standard boosting techniques (explained in Section~\ref{sec:ProofSketch}) we get the main result stated above. 

We detail the differences between our technique and prior CUR algorithms here. This includes additional assumptions required, algorithmic trade-offs, and discussion of sampling and computational complexity. 

\textbf{Block stable rank.} Theorem \ref{Thm2} tells us that the number of blocks required to achieve an $\varepsilon$ relative error depends on the structure of the blocks (through $\alpha_R$). Intuitively, this is saying the groups that provide more information improve the approximation faster than less informative groups. The $\alpha_R$ term depends on the stable or numerical rank (a stable relaxation of exact rank) of the blocks. The stable rank $\alpha = \|\bA\|^2_F/ \|\bA\|_2^2$ is a relaxation of the rank of the matrix; in fact, it is stable under small perturbations of the matrix $\bA$ \cite{stablerank}. For instance, the stable rank of an approximately low rank matrix tends to be low. The $\alpha_R$ term defined in Theorem \ref{Thm2} is the minimum stable rank of the column blocks. Thus, the $\alpha_R$ term gives a dependence of the block sampling complexity on the stable ranks of the blocks. It is easy to check that $1 \leq \alpha_R \leq s$. In the best case, when all the groups have full stable rank with equal singular values, $\alpha_R$ achieves its maximum. The worst case $\alpha_R = 1$ is achieved when a group or block is rank-$1$. That is, sampling groups of rank $s$ gives us a lot more information than groups of rank 1, which leads to a reduction in the total sampling complexity. 

\textbf{Incoherence.} The column space incoherence (Definition \ref{def:incoherence}) is used to provide a guarantee for approximation without computing the SVD of the entire matrix $\bA$. However, if it is possible to compute the SVD of the entire matrix, then the rows can be sampled using row leverage scores, and the incoherence assumption can be dropped. The relative error guarantee, independent of incoherence, for the full SVD Block CUR approximation is stated as Corollary 1 in the Appendix.  The corollary follows by a similar analysis as Theorem~\ref{Thm2} so we defer the proof to the Appendix. Other than block sampling, the setup of this result is equivalent to the traditional column sampling result stated in Lemma~\ref{Thm3}. Next, we compare the block sampling result with extensions of traditional column sampling. 

\begin{table}[t]
\centering
\caption{Table comparing the sample complexity needed for given $\varepsilon$ using our Block CUR result and a bound obtained by trivial extension of traditional CUR. For ease of comparison, we show the results with full SVD computation ignoring incoherence assumption stated in Corollary 1 in the Appendix. The $\alpha_R$ term we introduce satisfies the bound $ 1 \le \alpha_R  \le s$. 
}
\begin{tabular}{  c c c}
\hline
Results & $r$ & $g$  \\ \hline 
Traditional CUR extended to block setting & $\mathcal{O}\left(\frac{k^2}{\varepsilon^{2}}\log\left(\frac{1}{\delta}\right)\right)$ &   $\mathcal{O}\left(\frac{k^4}{\varepsilon^{6}}\log^3(\frac{1}{\delta})\right)$ \\  
 Our Block CUR & $\mathcal{O}\left(\frac{k^2}{\varepsilon^{2}}\log(\frac{1}{\delta})\right)$ & $\mathcal{O}\left(\frac{k^4}{\boldsymbol{\alpha_R} \varepsilon^{6}}\log^3(\frac{1}{\delta})\right)$ \\
 \hline   
\end{tabular}
\label{Tab:Complexity}
\end{table}

\textbf{Sample complexity: comparison with extensions of traditional CUR results.} In order to compare the sample complexity of our block sampling results with trivial block extensions of traditional column sampling results we focus our attention on the similar leverage score based CUR result in Theorem \ref{ThmCUR}. A simple extension to block setting could be obtained by considering a larger row space in which blocks are expanded to vectors. This would lead to a sample complexity bound obtained by Theorem \ref{ThmCUR}. The sampling complexity of the Block CUR derived in Theorem \ref{Thm2} tells us the number of sampling operations or queries that need to be made to memory in order to construct the $\bR$ and $\bC$ matrices.  As shown in Table \ref{Tab:Complexity} the column block sample complexity obtained by traditional CUR extensions results is always greater than or equal to those required by our Block CUR result because $ 1 \le \alpha_R  \le s$. This happens since traditional CUR-based results are obtained by completely ignoring the block structure of the matrix. 

As a side note, the authors are aware that more recent adaptive column sampling-based algorithms such as  \cite{wang12,woodruff14}  require only $c = O(k/\varepsilon)$ columns to be sampled. These results assume full computation of the SVD is possible, and they are byproducts of heavy machinery using ideas like deterministic, Batson/Srivastava/Spielman (BSS) sampling and adaptive sampling on top of leverage scores. By extending these advanced techniques to block sampling, it may be possible to obtain tighter bounds but it does not bring new insight into the problem of sampling blocks and unnecessarily complicates the discussion. Therefore we defer this extension to future work.

\subsection{Proof sketch of main result}\label{sec:ProofSketch}

In this section, we provide a sketch of the proof of Theorem \ref{Thm2} and defer the details to the Appendix. The proof of the main result rests on two important lemmas. These results are important in their own right and could be useful wherever the block sampling issue arises. The first result concerns approximate block multiplication. 

\textbf{Block multiplication lemma.} The following lemma shows that the multiplication of two matrices $\bA$ and $\bB$ can be approximated by the product of the smaller sampled and scaled block matrices. This is the key lemma in proving the main result.

\begin{lemma}\label{lemma4}
Let $\bA \in \R^{m \times n}$, $\bB \in \R^{n \times p}$,$\varepsilon, \delta \in (0,1)$, and $\alpha_A$ be defined as $\alpha_A := \min_{i \in [G]} \frac{\| \bA^{(i)}\|^2_F}{\| \bA^{(i)}\|^2_2} .$  Construct $\bC_{m \times gs}$ and $\bR_{gs \times n}$ using sampling probabilities $p_i$  that satisfy  $$ p_i \geq \beta \frac{\|\bA^{(i)}\|_F^2}{\sum_{j = 1}^{G}\|\bA^{(j)}\|_F^2},$$ for all $i \in [G]$ and where $\beta \in (0,1]$. Then, with probability at least $1 - \delta$, 
$$   \| \bA\bB - \bC\bR \|_F  \leq \frac{1}{\delta\sqrt{\beta g \alpha_A }}\|\bA\|_F\|\bB\|_F. $$
\end{lemma}

The proof details are provided in the Appendix. The main difficulty in proving this claim is to account for the block structure. Even though one could trivially extend individual column sampling analysis to this setting by serializing the blocks, this would lead to trivial bounds as they do not leverage the block structure. Our results exploit this knowledge and hence introduce a dependence of the sample complexity on the block stable rank of the matrix.

Using the block multiplication lemma we prove Lemma~\ref{Thm3}, which states a non-boosting approximation error result for Algorithm~\ref{groupalgo}.

\begin{lemma}\label{Thm3}
Given $\bA \in \R^{m \times n}$ with incoherent top-$k$ column space, i.e.,~$\mu \leq \mu_0$, let $r = O(\mu_0\frac{k^2}{\varepsilon^2}) $  and $g = O(\frac{ r^2}{\alpha_R\varepsilon^2} )$. If rows and column blocks are chosen according to Algorithm \ref{groupalgo}, then with probability at least 0.7, the following holds:
 \begin{equation*} \| \bA - \bC \bU \bR\|_F \leq (1 + \varepsilon ) \| \bA -  \bA_k \|_F, \end{equation*}
  where $\varepsilon \in (0,1)$, $\bU = \bW^{\dagger}$ is the pseudoinverse of the scaled intersection of $\bC$ and $\bR$.
\end{lemma}

The proof of Lemma~\ref{Thm3} follows standard techniques in \cite{drineas08} with modifications necessary for block sampling (see Appendix for the proof details). Finally, the result in Theorem \ref{Thm2} follows by applying standard boosting methods to Lemma~\ref{Thm3} and running Algorithm~\ref{groupalgo} $t = \ln(\frac{1}{\delta})$ times. By choosing the solution with minimum error and observing that $0.3 < 1/e$, we have that the relative error bound holds with probability greater than $1 - e^{ - t} = 1 - \delta $.

\textbf{Remark.} As a consequence of Lemma \ref{Thm3}, we show that if enough blocks are sampled with high probability, then $ \| \bA - \bA\bS(\bR\bS)^{\dagger}\bR\|_F $ $ \leq $ $ (1 + \varepsilon ) \| \bA - \bA\bR^{\dagger}\bR\|_F$. 
This gives a guarantee on the approximate solution obtained by solving a block-sampled regression problem $ \min_{\bX \in \R^{m \times r}}\| (\bA\bS) - \bX(\bR\bS)\|_F$ instead of the entire least squares problem.  As a special case of the above result, when $\bR = \bA$ we get a bound for the block column subset selection problem. If $g $  $ = $ $ \mathcal{O}(\frac{ k^2}{\alpha_A \varepsilon^2} $ $\log\left(\frac{1}{\delta}\right))$ blocks are chosen, then with probability at least $1 - \delta$ we have $ \| \bA - \bC\bC^{\dagger}\bA\|_F $ $ \leq $ $ (1 + \varepsilon ) \| \bA - \bA_k\|_F. $

\section{Experiments}

\subsection{Experiments with biometric data}
\label{sec:eda_experiments}

One emerging application is audience reaction analysis of video content using biometrics.  Specifically, users watch video content while wearing sensors, with changes in biometric sensors indicating changes in reaction to the content.  For example, increases in heart rate or a spike in electrodermal activity indicate an increase in content engagement.  In prior work, biometric signal analysis techniques have been developed to determine valence~\cite{silveira2013predicting} ({\em e.g.,} positive vs. negative reactions to films) and content segmentation~\cite{lian2014}.  Unfortunately these experiments require a large number of users to sit through the entire video content, which can be both costly and time-consuming.  

We consider the observed biometric signals as a matrix with $m$ users (as rows) and $n$ biometric time samples (as columns).  Matrix approximation techniques, such as CUR decomposition, point to the ability to infer the complete matrix by showing the entire content to only a subset of users ({\em i.e.,} rows), while the remaining users see only selected scenes of the content ({\em i.e.,} column blocks).  To replicate a user's true reaction to content, individual columns cannot be sampled ({\em e.g.,} showing the user 0.25 seconds of video content) given the lack of scene context. Instead, longer scenes must be shown to the user to gather a representative response.  Therefore, the Block CUR decomposition proposed in this paper is directly applicable.  

The biometric experiment setup is as follows.  We attached $24$ subjects with the Empatica E3 wearable sensor~\cite{empatica} that measures electro-dermal activity (EDA) at $4$ Hz.  The subjects were shown a $41$-minute episode of the television series ``NCIS'', in the genres of action and crime.  The resulting biometric data matrix was $24 \times 9929$.  Our goal is to use Block CUR decomposition to show only a subset of users the entire content, and to then impute the biometric data for users that have viewed only a small number of selected scenes from the content.
 \begin{figure}[tp]
	\centering
	\subfloat[]{
		\label{fig:EDA}
	    \includegraphics[width=0.4\textwidth]{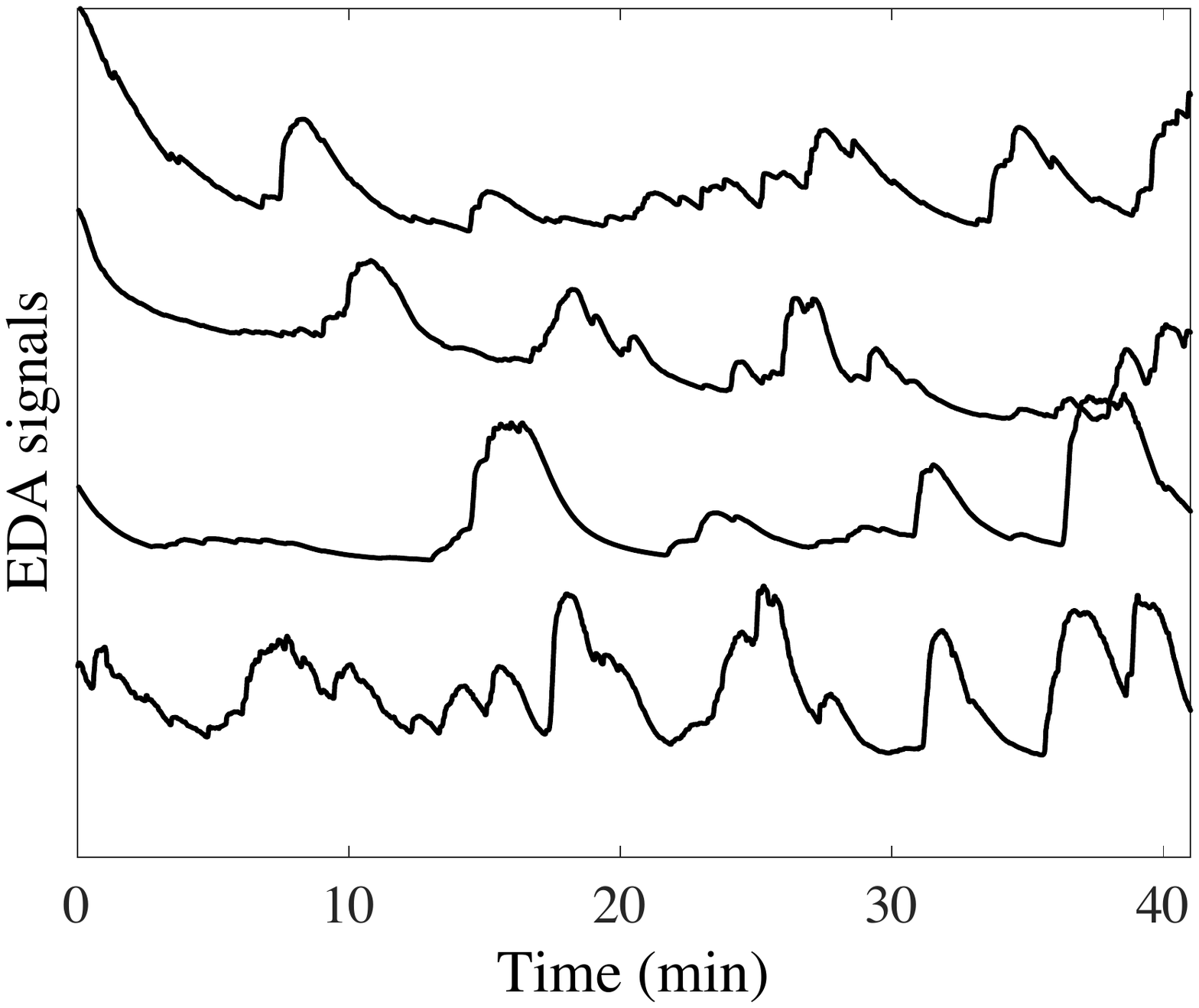}
	}
	\qquad
	\subfloat[]{
		\label{fig:svdfro}
		\includegraphics[width=0.43\textwidth]{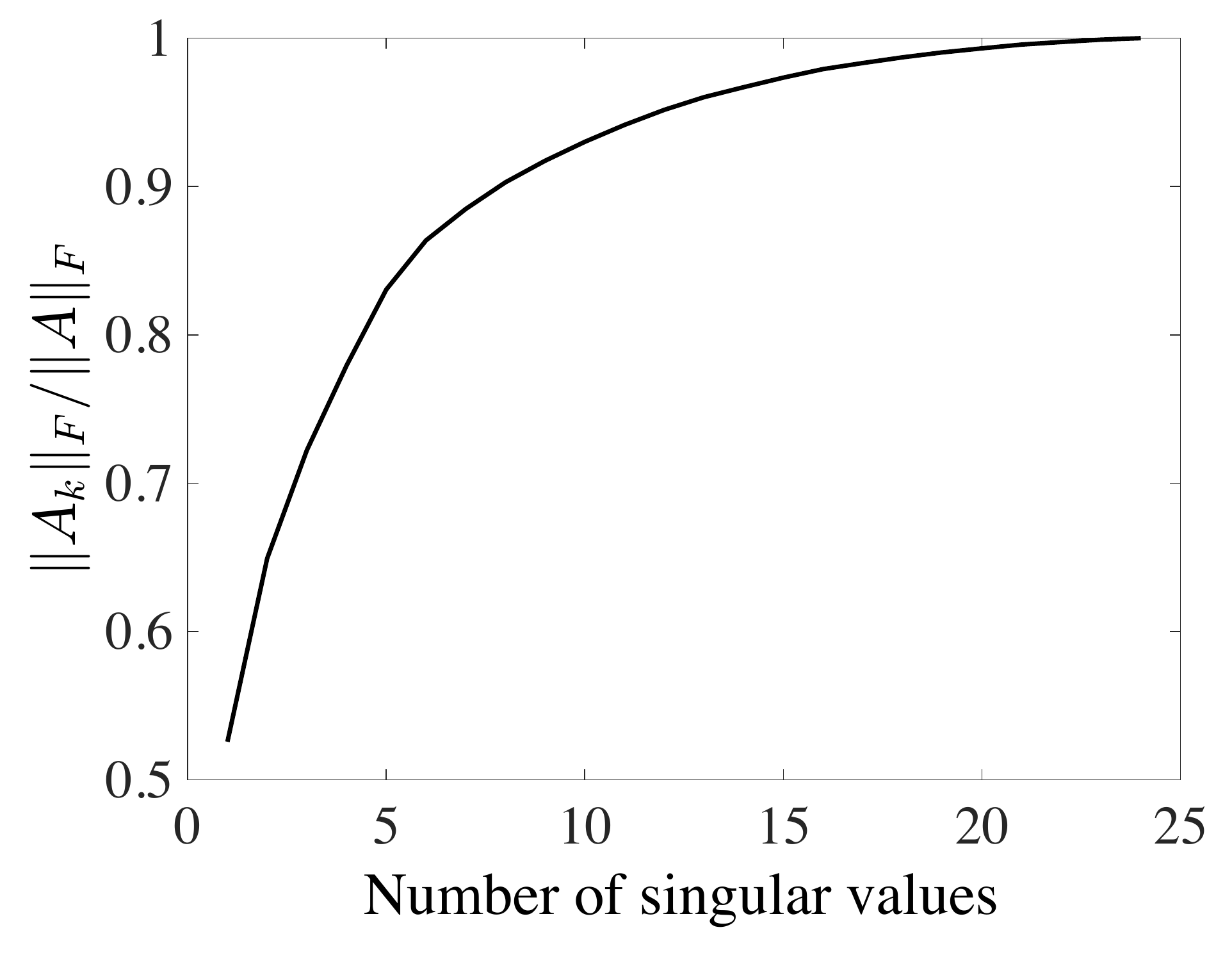}
	}
	\caption[]{Panel \subref{fig:EDA} shows EDA data for four users watching the NCIS video and \subref{fig:svdfro} demonstrates the low rank nature of $\bA$ }
\end{figure}

\begin{figure}[tp]

 \center

  \includegraphics[width=0.45\textwidth]{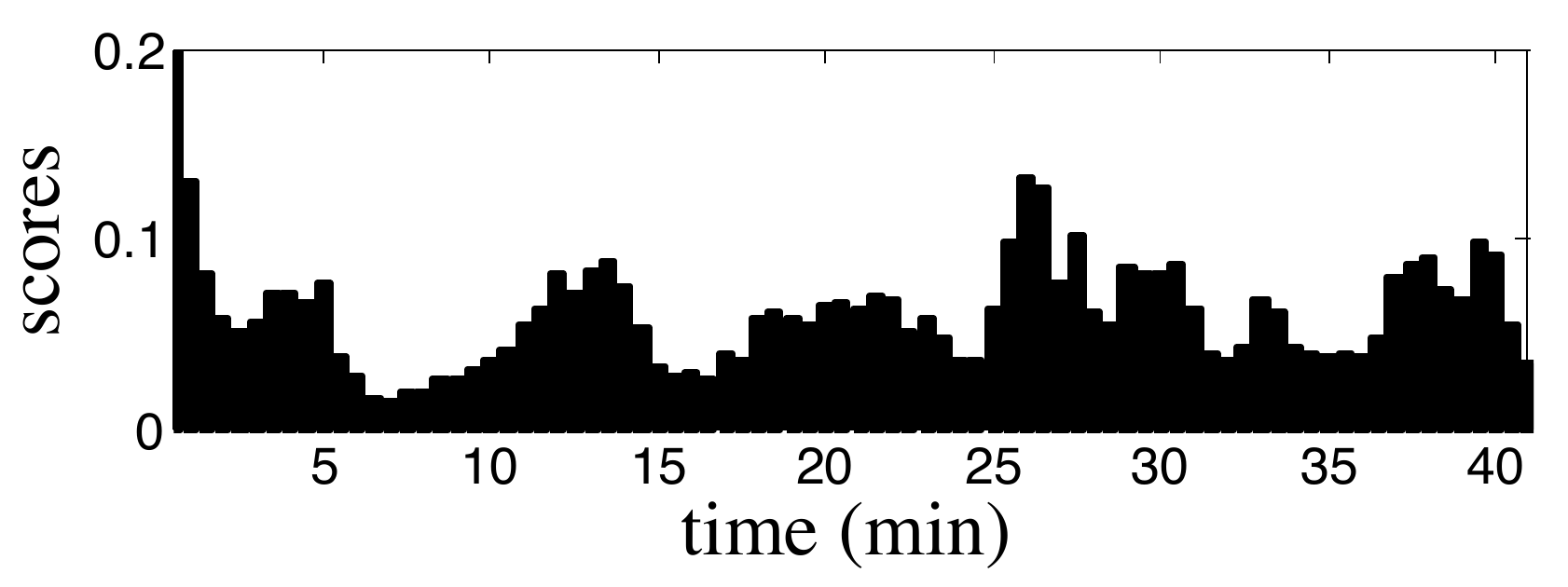}

  \caption{Block leverage scores for EDA data with k = 5 and s = 120 columns (30 seconds).}

  \label{Fig:leverage}

\end{figure}

\begin{figure}[tp]
	\centering
	\subfloat[$s=60$ using $\bU$]{
	          \includegraphics[width=0.4\textwidth]{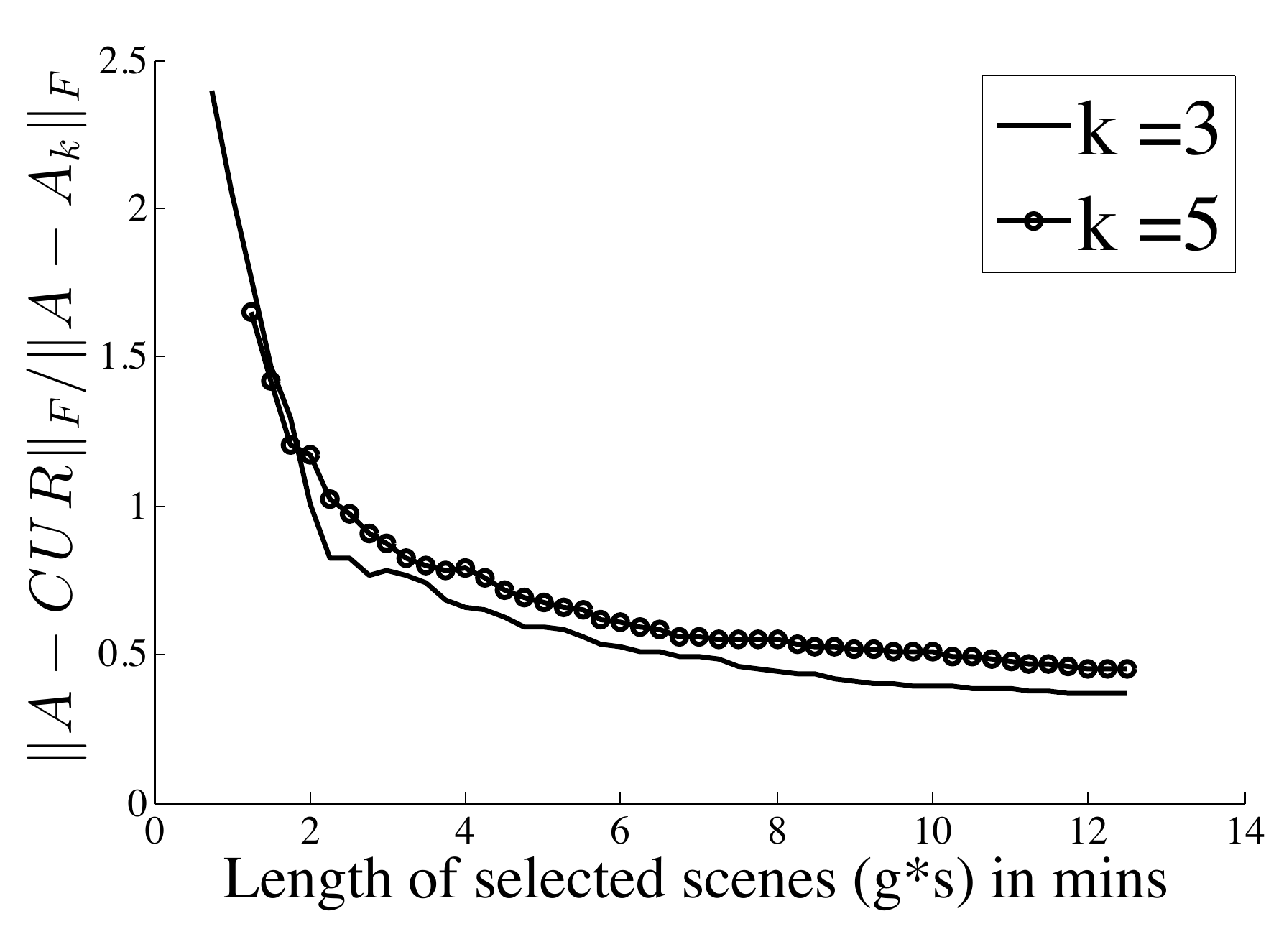}
	}
	\quad
	\subfloat[$s=60$ using $\bU_k$]{
		\includegraphics[width=0.4\textwidth]{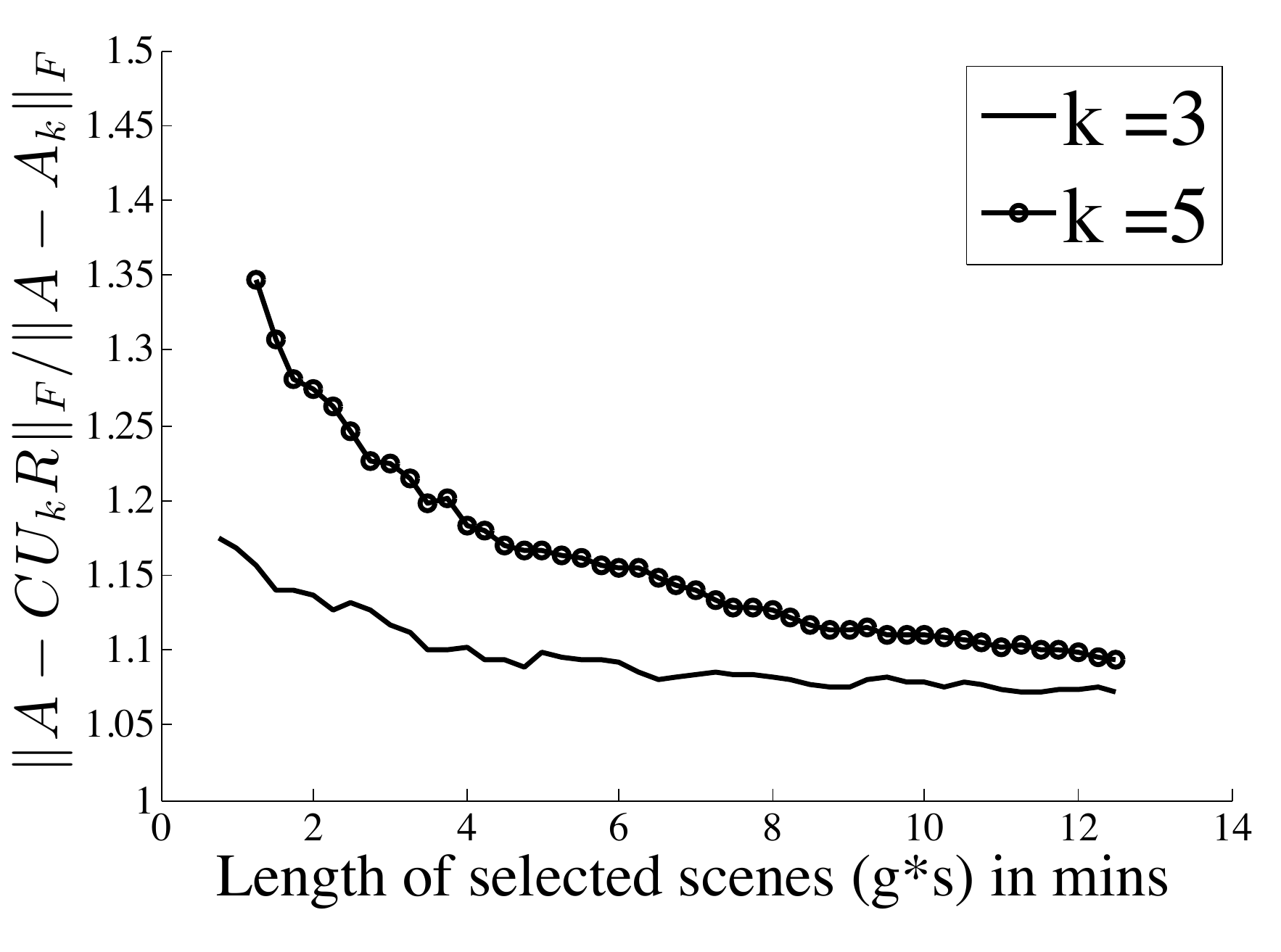}
	}
	\quad
	\subfloat[$s=120$ using $\bU$]{
		\includegraphics[width=0.4\textwidth]{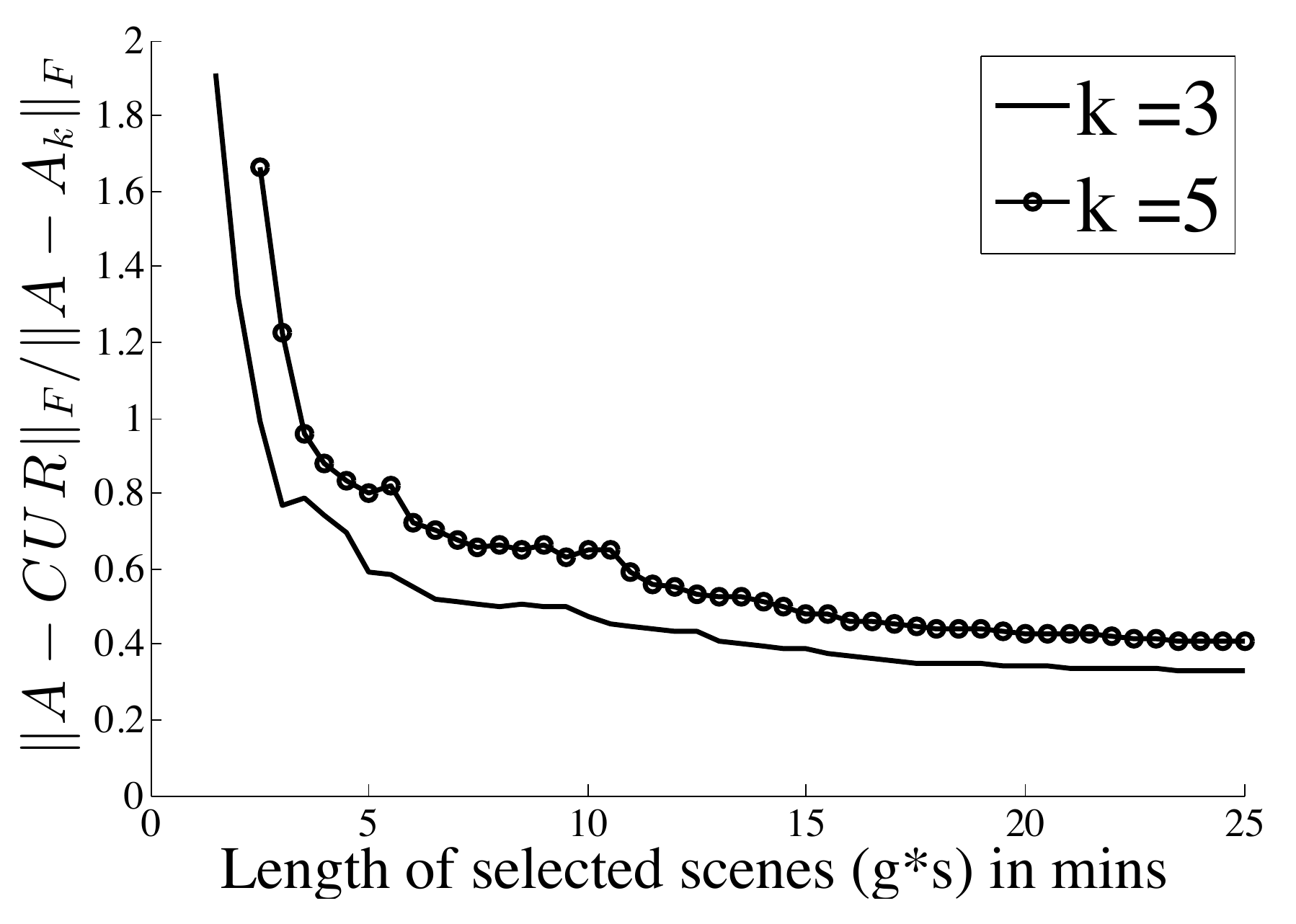}
	}
	\quad
	\subfloat[$s=120$ using $\bU_k$]{
	         \includegraphics[width=0.4\textwidth]{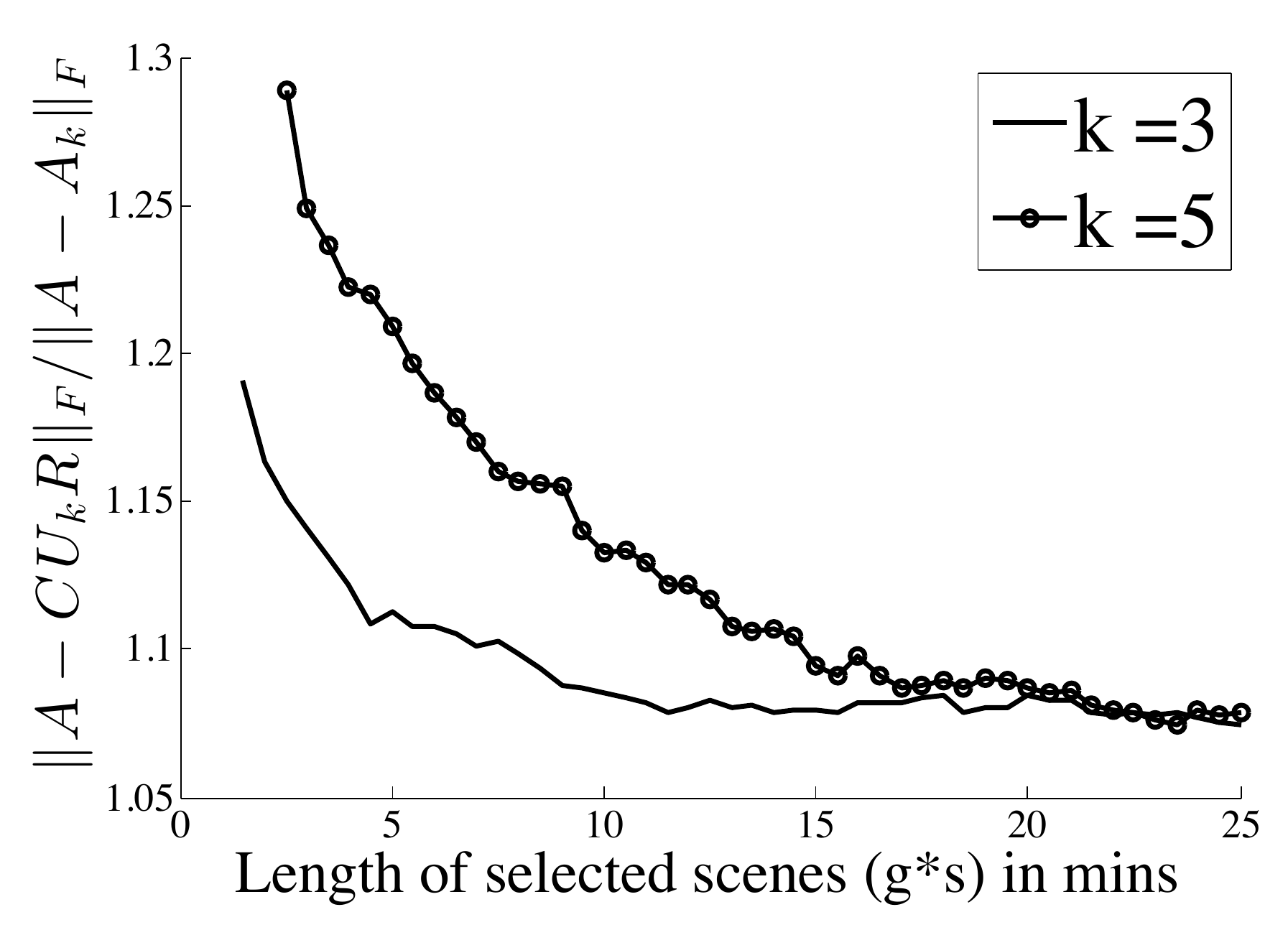}
	}
	\caption{Error plots for two values of target rank, $k = 3,5$.}
	\label{Fig:error}
\end{figure}

\textbf{Results.} We refer to the biometric data matrix as $\bA$ and plot the EDA traces (rows) corresponding to four users in Figure \ref{fig:EDA}.  To demonstrate the low rank nature of the data, we plot the Frobenius norm of $\bA$ covered by $\bA_k$ as a function of $k$ in Figure \ref{fig:svdfro}. We find that for this data, only $5$ singular vectors are needed to capture $80\%$ of the total Frobenius norm of the complete matrix. Next, we segment the columns of this matrix into blocks such that $s = 120$ columns (or $30$  seconds).  In Figure \ref{Fig:leverage}, we show the computed block leverage scores. The leverage scores seem to suggest that certain scenes are more important than others. For example, the highest leverage scores are around the 12, 26, and 38 minute marks. This corresponds to scenes of a dead body, unveiling of a clue to solving the mystery, and the final arrest, respectively.  

Using Algorithm \ref{groupalgo}, we uniformly sample EDA traces (rows) of 20 users and hold out the EDA traces of 4 users. We sample column blocks and plot the resulting error in Frobenius norm in Figure \ref{Fig:error}.  The plots show the normalized Frobenius norm error of the CUR approximation as a function of the number of blocks, $g$, sampled. More precisely, the ratio $\|\bA - \bC\bU\bR\|_F/\|\bA - \bA_k\|_F$ and $\|\bA - \bC\bU_k\bR\|_F/\|\bA - \bA_k\|_F$ are plotted for two values of the target rank, $k = 3$ and $5$ and two values of block size, $s = 60$ and $120$ columns per block ($15$ and $30$ seconds), respectively. We also compare the error using $\bU_k$, the rank-$k$ approximation of $\bU$, which leads to an exactly rank-$k$ matrix approximation since this may be a restriction in some applications. We repeat Algorithm \ref{groupalgo} ten times\footnote{These plots were generated using \textit{sampling without replacement} even though our theory supports \textit{sampling with replacement} since sampling the same blocks is inefficient in practice.} and plot the mean normalized error over 10 trials. 

The error drops sharply as we sample more blocks but quickly flattens demonstrating that a summary of the movie could suffice to approximate the full responses. The plots also show the interplay between the number of blocks sampled and the issue of context which is related to block size. To give the viewer some context we would want to make the scene as long as possible but we want to show them only a summary of the content to reduce the cost. These conflicting aims result in a trade-off of block size and the number of blocks sampled. For example, for $k = 5$, the normalized error is less than $1$ when a $2.5$ minute long clip is shown to the viewer, that is $g = 10$ with block size $s = 60$ columns (or $15$ seconds), whereas the normalized error is less than $1$ when a $3.5$ minute long clip is shown to the viewer ($g = 7$) with block size $s = 120$ columns (or $30$ seconds).  These results demonstrate the practical use of the Block CUR algorithm.

\subsection{Distributed experiments}
\label{sec:experiments}

In this section we demonstrate empirically that the proposed block sampling based CUR algorithms can achieve a significant speed-up when used to decompose matrices in a distributed data setting by comparing their performance with individual column sampling based traditional CUR algorithms on both synthetic and real-world data. We report the relative-error of the decomposition ({\em i.e.,} $\|\bA - \bC\bU\bR\|_F/\|\bA\|_F $) and the sampling time of each algorithm on different data-sets. 

\begin{table}[t]
\centering
\caption{Table comparing the number of sampling operations needed for given $\varepsilon$ using our Block CUR result based on block sampling and traditional CUR based on individual column sampling (note this is not the same as the vectorized block columns in Table 1). This leads to speedup since it is more efficient to retrieve predefined blocks than querying individual rows or columns in these regimes. The $\alpha_R$ term we introduce satisfies the bound $ 1 \le \alpha_R  \le s$.}
\begin{tabular}{ c c }
\hline
Method  & No. of sampling ops.  \\ \hline 
Traditional CUR  &  $\mathcal{O} \left( \frac{k^2}{\varepsilon^{2}}\log(\frac{1}{\delta}) + \frac{k^4}{\varepsilon^{6}} \log^3(\frac{1}{\delta})\right)$\\  
Block CUR  & $\mathcal{O}\left(  \frac{k^2}{\varepsilon^{2}}\log(\frac{1}{\delta}) + \frac{  k^4}{\boldsymbol{\alpha_R} \varepsilon^{6}}\log^3(\frac{1}{\delta})\right)$\\
 \hline   
\end{tabular}
\label{Tab:Latency}
\end{table}

We implemented the algorithms in Scala 2.10 and Apache Spark 2.11 on Amazon Elastic Map-Reduce (Amazon EMR). The compute cluster was constructed using four Amazon m4.4xlarge instances, with each compute node having 64 GB of RAM. Using Spark, we store the data sets as resilient distributed dataset (RDD), a collection of elements partitioned across the nodes of the cluster (see Figure~\ref{fig:blockCUR}). In other words, Spark partitions the data into many blocks and distributes these blocks across multiple nodes in the cluster. Using block sampling, we can approximate the matrix by sampling only a subset of the important blocks.  Meanwhile, individual column sampling would require looking up all the partitions containing specific columns of interest as shown in Table \ref{Tab:Latency}. Our experiments examine the runtime speed-up from our block sampling CUR that exploits the partitioning of data.

\textbf{Synthetic experiments.}
The synthetic data is generated by $\bA = \bU\bV$ where $\bU \in \R^{m \times k}$ and $\bV \in \R^{k \times n}$ are random matrices with i.i.d. Gaussian random entries, resulting in a low rank matrix $\bA$.  We perform CUR decomposition on matrices of size $m \times n$ with $m=n$, target rank $k$, and number of blocks $G$ (set here across all experiments to be $100$). The leverage scores are calculated by computing the SVD of the rows sampled uniformly with $\bR \in \R^{r \times n}$. We sample one-sixth of the rows. 

\begin{figure}[t]
\centering
	
	\subfloat[$n = 8000$]{
	  \includegraphics[width=0.31\textwidth]{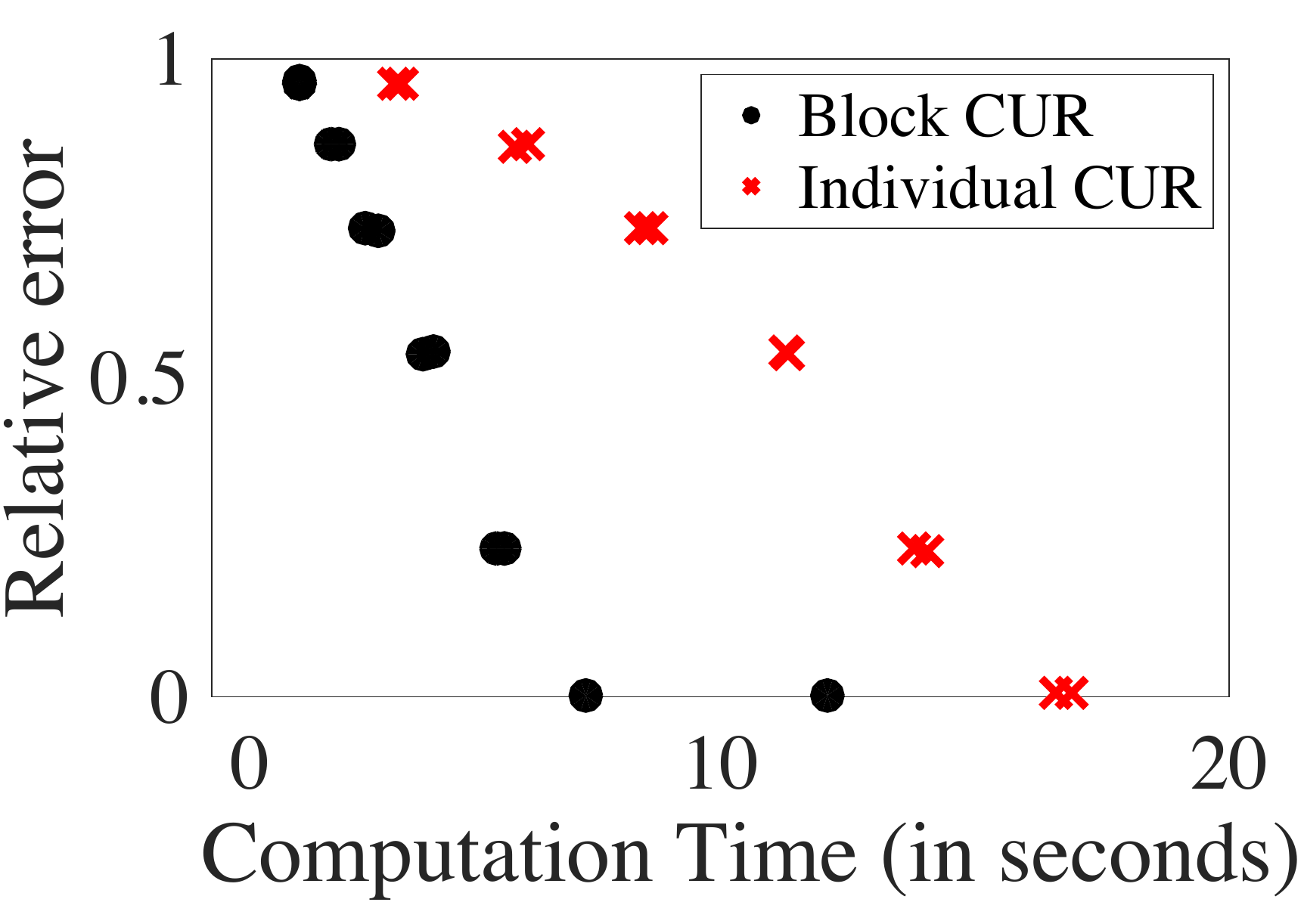} 
	}
	\,
	\subfloat[$n = 12000$]{
	  \includegraphics[width=0.31\textwidth]{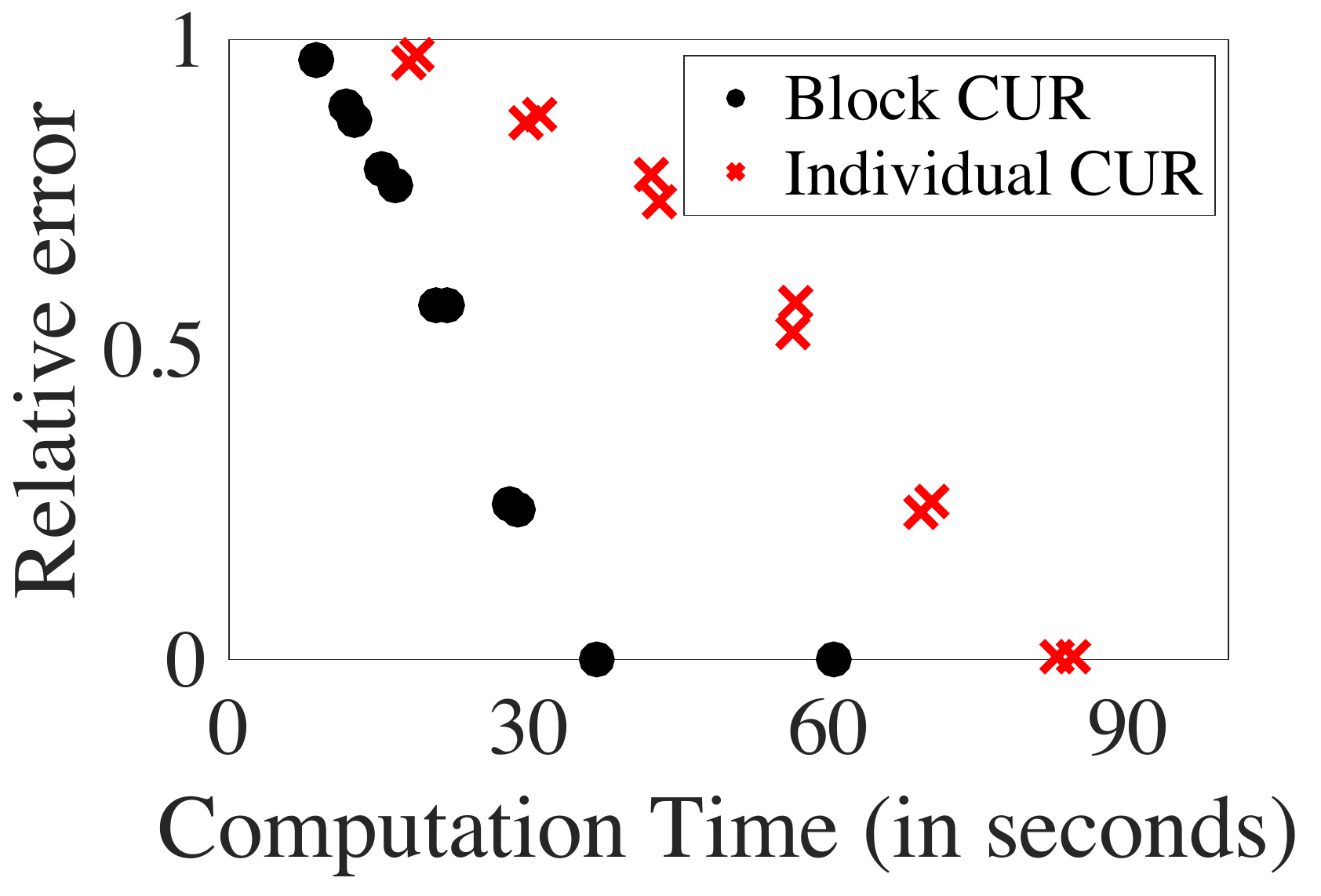}
	}
	\,
	\subfloat[$n = 20000$]{
	  \includegraphics[width=0.31\textwidth]{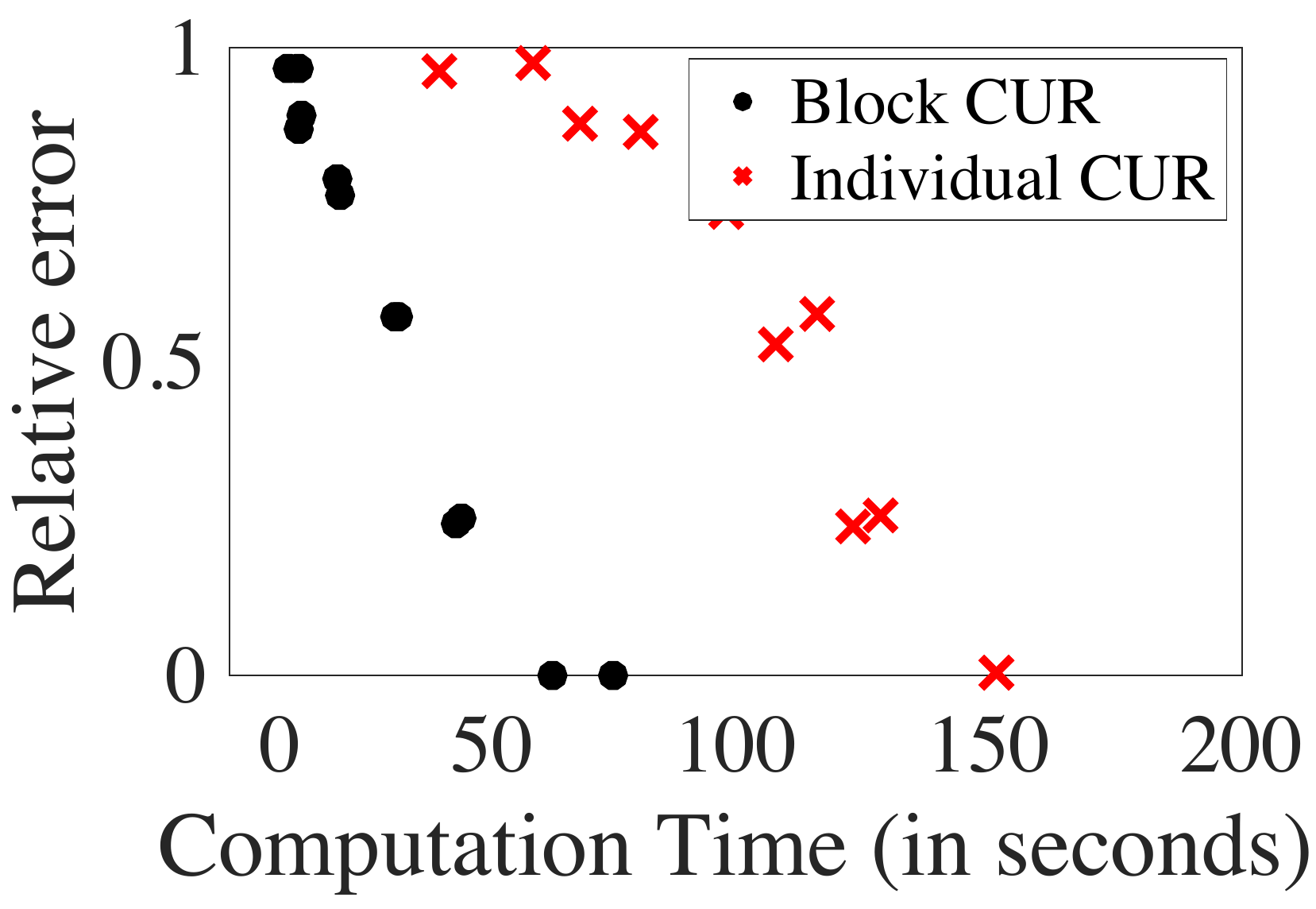}
	}
 	
\caption{Performance on synthetic $n \times n$ matrices with rank $n/10$.}
 \label{Fig:synthetic}
\end{figure}

Figure \ref{Fig:synthetic} shows the plots for relative error achieved with respect to the runtime required to sample $\bC$ and $\bR$ matrices for both Block CUR and traditional CUR algorithms. To focus on the speed-up achieved by taking into account the block storage of data we compare running times of only the sampling operations of the algorithms (which excludes the time required to compute the SVD). We note that other steps in both algorithms can be updated to include faster variants such as the approximation of leverage scores by sketching or sampling \cite{approxLevScores}. We vary $g$, the number of blocks chosen, from $1$ to $6$. The number of columns chosen is thus $c=gs$, where $s$ denotes the number of columns in a block and varies from $50$ to $200$. We repeat each algorithm (Block CUR and traditional CUR) twice for the specified number of columns, with each realization as a point in the plot. The proposed Block CUR algorithm samples the $c$ columns in $g$ blocks, while traditional CUR algorithm samples the $c$ columns one at a time.

Consistently, these results show that block sampling achieves the relative error much faster than the individual column sampling -- with performance gains increasing as the size of the matrix grows, as shown in Figure \ref{Fig:synthetic}. While the same amount of data is being transmitted regardless of whether block or individual column sampling is used, block sampling is much faster because it needs to contact fewer executors to retrieve blocks of columns rather than the same number of columns individually. In the worst case, sampling individual columns may need to communicate with all of the executors, while block sampling only needs to communicate with $g$ executors. Thus, by exploiting the partitioning of the data, the Block CUR approach is able to achieve roughly the same quality of approximation as traditional column-based CUR, as measured by relative error, with significantly less computation time.

\begin{figure}[t]
	\centering
	\subfloat[]{
		\label{fig:arcene_speedup}
	    \includegraphics[width=0.47\textwidth]{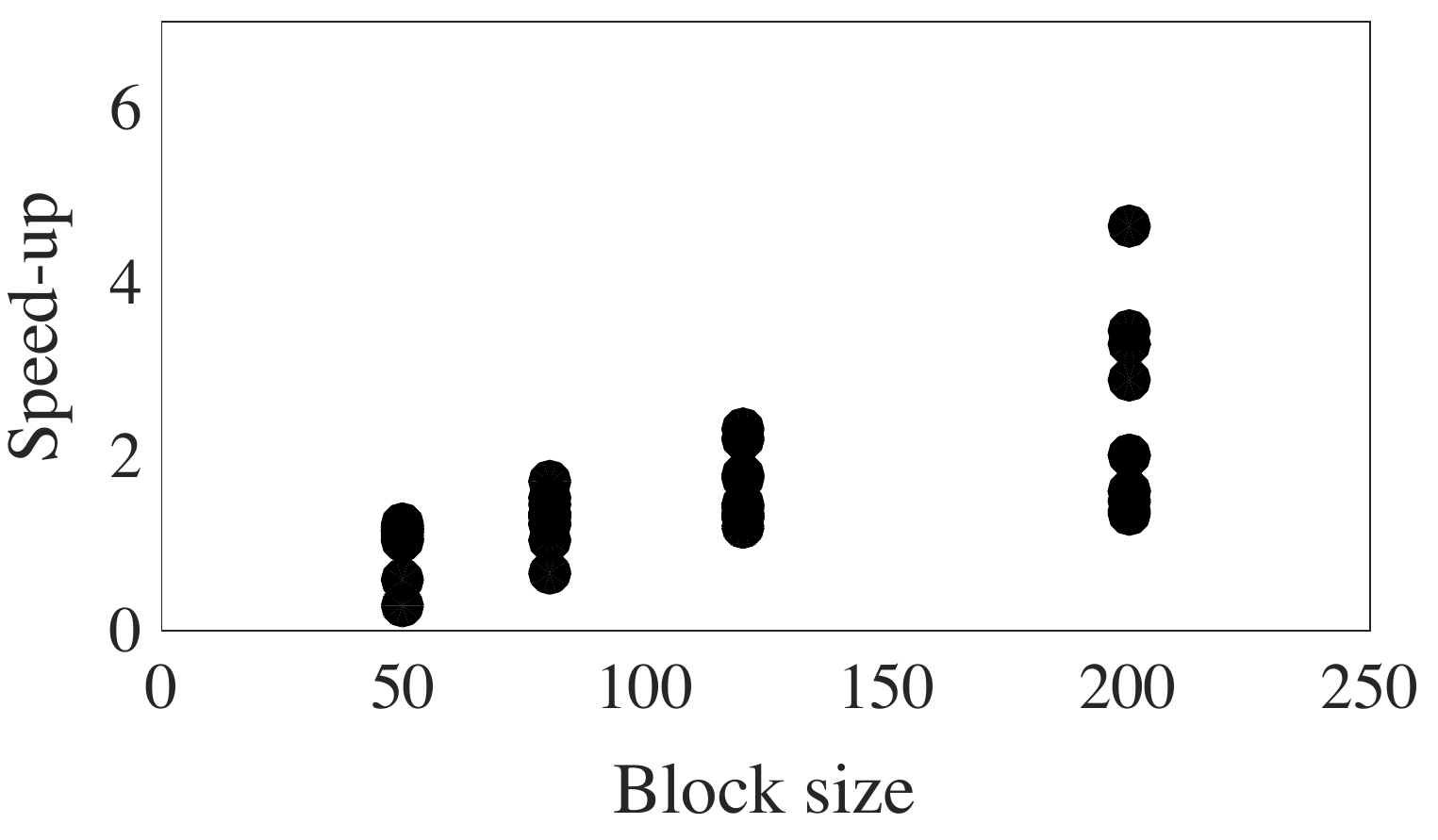}
	}
	\,
	\subfloat[]{
		\label{fig:arcene_error}
		\includegraphics[width=0.47\textwidth]{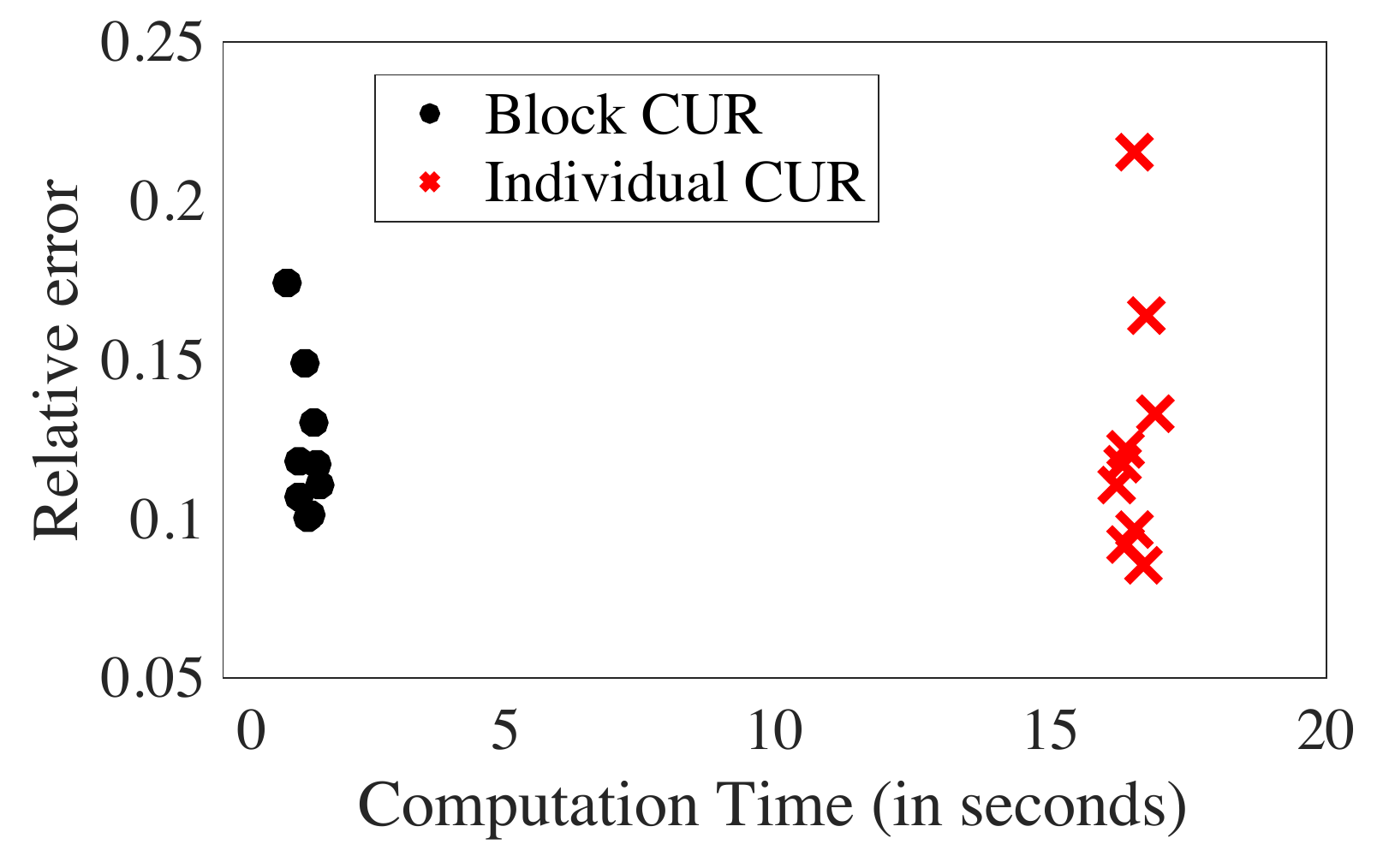}
	}
	\caption[]{Performance on $900 \times 10,000$ Arcene dataset with block size 12. \subref{fig:arcene_speedup} Runtime speed-up from block sampling compared to individual column sampling for varying block sizes. \subref{fig:arcene_error} Block CUR achieves similar relative errors as individual CUR with much lower computation time.}
	\label{fig:arcene}
\end{figure}

\textbf{Real-world experiments.} We also conduct experiments on the Arcene dataset \cite{arcene} which has 900 rows and 10,000 columns. We compare the running time for both block and traditional CUR decomposition.  We again find consistent improvements for the block-wise approach compared with individual column sampling. With block size $s = 12$, sampling up to 10 groups led to an average speed up of 11.2 over individual column sampling, as shown in Figure \ref{fig:arcene}. The matrix is very low rank, and sampling a few groups gave small relative errors.

\section{Conclusion}
\label{sec:conclude}

In this paper we extended the problem of CUR matrix decomposition to the block setting which is naturally relevant to distributed storage systems and biometric data analysis. We proposed a novel algorithm and derived its performance bounds. We demonstrated its practical utility on real-world distributed storage systems and audience analytics. Some possible future directions for this work include calculating the leverage scores quickly or adaptively, and considering the algorithms and error bounds when the matrix has a pre-specified structure like sparsity. 

\bibliographystyle{abbrv}
\bibliography{root}

\appendix
\section{Supplementary Material for Block CUR: Decomposing Matrices using Groups of Columns}
\label{sec:theory}

As stated in Section~\ref{sec:ProofSketch}, the result in Theorem \ref{Thm2} follows by applying standard boosting methods to Lemma~\ref{Thm3} and running Algorithm 1 $t = \ln(\frac{1}{\delta})$ times. By choosing the solution with minimum error and observing that $0.3 < 1/e$, we have that the relative error bound holds with probability greater than $1 - e^{ - t} = 1 - \delta $.  Hence, it suffices to prove Lemma~\ref{Thm3} to prove the main result. 

\subsection{Proof of Lemma ~\ref{Thm3}}

First, note $\bU = (\bR\bS)^{\dagger}$ and $\bC = \bA\bS$. 
\begin{align*}
\|\bA - \bC\bU\bR\|_F &= \|\bA - \bA\bS (\bR\bS)^{\dagger}\bR\|_F 
\end{align*}

Recall that $\bR \in \R^{r \times n}$ has rank no greater than $r$;  $\bA \in \R^{m \times n}; \varepsilon \in (0,1)$; and that the same column blocks from $\bR$ and $\bA$ are picked with the following probability distribution:
$$ p_i = \frac{\|\bV_{R,r}^T \bE_i\|_F^2}{r}, \ \forall i \in [G]. $$ 

We can use Lemma~\ref{intermediateresult} (stated and proved in next section) with probability at least 0.85 we have
$$ \| \bA - \bA\bS(\bR\bS)^{\dagger}\bR\|_F \leq (1 + \varepsilon ) \| \bA - \bA\bR^{\dagger}\bR\|_F. $$

Next,  we bound $\|\bA - \bA\bR^{\dagger}\bR \|_F$. Since $\bA$ has incoherent column space, the uniform sampling distribution $p_j = 1/m$ satisfies eqn. (13) in \cite{drineas08} with $\beta = 1/\mu_0$. Consequently, we can apply modified version of Theorem 1 in \cite{drineas08} we get with probability at least 0.85,
$\|\bA - \bA\bR^{\dagger}\bR \|_F \leq (1 + \varepsilon) \| \bA - \bA_k\|_F$. Finally, we get with probability 0.7,
\begin{align*}
\|\bA - \bC\bU\bR\|_F & \leq (1 + \varepsilon' )^2 \| \bA - \bA_k \|_F,&\\
& \leq (1 + \varepsilon'' ) \| \bA - \bA_k \|_F,&\textnormal{ letting } \varepsilon'' = 3\varepsilon'.
\end{align*}
This completes the proof of Lemma 1.

\subsubsection{Approximating generalized $\ell_2$ regression in the block setting}\

In this section, we give theory for generalized least squares using block subset selection that is used to prove the main results for the algorithms but applies to arbitrary matrices $\bA$ and $\bB$.  Given matrices $\bA \in \R^{m \times n}$ and  $\bB \in \R^{r \times n}$, the generalized least squares problem is
$$ \min_{\bX \in \R^{m \times r}}\| \bA - \bX\bB\|_F.$$
It is well-known that the solution to this optimization problem is given by $ \widehat{\bX} = \bA\bB^{\dagger}$. To approximate this problem by a subsampled problem, we sample some blocks of columns from $\bA$ and $\bB$ to approximate the standard $\ell_2$ regression by the following optimization: 
$$ \min_{\bX \in \R^{m \times r}}\| (\bA\bS) - \bX(\bB\bS)\|_F.$$
The solution of this problem is given by $ \tilde{\bX} = \bA\bS(\bB\bS)^{\dagger}$. In the following lemma, we give a guarantee stating that, when enough blocks are sampled with the specified probability, the approximate solution is close to the actual solution to the $\ell_2$ regression.  

\begin{lemma}\label{intermediateresult}
	Suppose $\bB \in \R^{r \times n}$ has rank no greater than $k$;  $\bA \in \R^{m \times n}; \varepsilon,\delta \in (0,1)$; and let the same column blocks from $\bB$ and $\bA$ be picked with the following probability distribution:
	$$ p_i = \frac{\|(\bV_{B,k})_{(i)}\|_F^2}{k}, \ \forall i \in [G]. $$ 
	
	If $g =  \mathcal{O}(\frac{ k^2}{\alpha_B \delta^4\varepsilon^2})$ blocks are chosen, then with probability at least $1 - \delta$ we have
	$$ \| \bA - \bA\bS(\bB\bS)^{\dagger}\bB\|_F \leq (1 + \varepsilon ) \| \bA - \bA\bB^{\dagger}\bB\|_F. $$
\end{lemma}

\begin{proof}
	
	Let $\bB = \bU_k \bSig_k \bV_k^T$ and $\alpha = \max_i\left( \frac{\| \bV_k^TE_i\|_2 }{\| \bV_k^TE_i\|_F} \right)^2$
	
	We start by showing $\bV_k^T\bS$ is full rank. Using Lemma 2, if $g \geq  8 \alpha_R^{-1} k^2 \delta^{-2} \varepsilon_1^{-2}$ and $0 < \varepsilon_1 < 1$, we get the following with probability $\geq 1 - \delta_1$,
	$$ \| \bV_k^T \bV_k - \bV_k^T\bS \bS^T \bV_k \|_2 = \| \bI_k - \bV_k^T\bS \bS^T \bV_k \|_2 \leq 4 \frac{k }{\delta \sqrt{\alpha_R g}} \leq \frac{\varepsilon_1}{2}.$$

	This further gives us a bound on the singular values of $\bV_k^T \bS$,  for all $i$, 
	
	\begin{equation}\label{eqn:sig}| 1 - \sigma_{i}^2(\bV_k^T \bS)| = | \sigma_{i}(\bV_k^T\bV_k) - \sigma_{i}(\bV_k^T \bS \bS^T\bV_k)| \leq \| \bI_k - \bV_k^T\bS \bS^T \bV_k \|_2  \leq \varepsilon_1.\end{equation}
	
	Thus, it follows for all singular values of $\bV_k^T\bS$,
	
	\begin{equation} \label{eqn:sigma}\sqrt{1 - \varepsilon_1} \leq \sigma_i(\bV_k^T\bS) \leq \sqrt{1 + \varepsilon_1}. \end{equation}

	Now, consider
	\begin{align*} 
	\|\Omega \|_2 & = \|(\bV_k^T \bS)^{\dagger} - (\bV_k^T \bS)^T\|_2 \\
	& = \|\bSig_{\bV_k^T \bS}^{-1} - \bSig_{\bV_k^T \bS} \|_2 \\
	& = \max_{i} \left| \sigma_i(\bV_k^T \bS) -  \frac{1}{\sigma_i(\bV_k^T \bS)} \right|\\
	& = \max_{i} \frac{| \sigma_i^2(\bV_k^T \bS) -  1 |}{|\sigma_i(\bV_k^T \bS)|}  \\
	& \leq \frac{\| \bV_k^T \bV_k - \bV_k^T\bS \bS^T \bV_k \|_2}{\sqrt{1- \| \bV_k^T \bV_k - \bV_k^T\bS \bS^T \bV_k \|_2}}  \\
	& \leq \frac{\varepsilon_1/2}{\sqrt{1 - \varepsilon_1/2}}\\
	& \leq  \varepsilon_1/\sqrt{2},
	\end{align*}
	where the first inequality follows from equation (\ref{eqn:sig}), the second inequality follows by applying Lemma~\ref{lemma4} and the last inequality follows since $\varepsilon_1 < 1$ implies $\sqrt{1 - \varepsilon_1/2} > 1/\sqrt{2}$
	
	Also, for any  $\bQ $ we have,  
	\begin{align*} 
	\E [\| \bQ \bS\|_F^2] &= \E \left[\sum_{t = 1}^{g} \left\|\frac{1}{\sqrt{gp_{j_t}}}\bQ^{(j_t)}\right\|_F^2\right]\\
	& = \sum_{t = 1}^{g}  \E \left[\frac{1}{g p_{j_t}}\left\| \bQ^{(j_t)} \right\|_F^2\right] \\
	& = \sum_{t = 1}^{g}  \sum_{i = 1}^{G} p_i \frac{1}{g p_i}\left\|\bQ^{(i)} \right\|_F^2 \\
	&= \|\bQ\|_F^2. 
	\end{align*}
	By Jensen's inequality,
	\begin{equation*} 
	\E [\| \bQ \bS\|_F]^2 \leq \E [\| \bQ \bS\|_F^2] = \|\bQ\|_F^2. 
	\end{equation*}
	By applying Markov's inequality, we get with probability $\geq 1- \delta'$, 
	\begin{equation} \label{eqnS} \| \bQ \bS\|_F \leq \frac{1}{\delta'} \E [\| \bQ \bS\|_F]  \leq \frac{1}{\delta'}\| \bQ\|_F. \end{equation}
	
	
	The following will be useful later, 
	\begin{align*}
	\bA\bS(\bB\bS)^{\dagger}\bB  &=  \bA\bS(\bU_k \bSig_k \bV_k^T\bS)^{\dagger}\bU_k \bSig_k \bV_k^T \\
	& =  \bA\bS( \bV_k^T\bS)^{\dagger}\bSig_k^{-1}\bU_k ^T \bU_k  \bSig_k \bV_k^T  \\
	& =  \bA\bS (\bV_k^T\bS)^{\dagger}  \bV_k^T 
	\end{align*}
	
	Using this result and observing that $(\bV_k\bV_k^T + \bV_k^{\perp} \bV_k^{\perp T} ) = \bI$, we break down the left hand term into 3 manageable components,
	\begin{align*}
	&\| \bA - \bA\bS(\bB\bS)^{\dagger}\bB\|_F \\
	&= \| \bA - \bA\bS(  \bV_k^T\bS)^{\dagger}  \bV_k^T\|_F \\
	&= \| \bA - \bA\bV_k\bV_k^T \bS(  \bV_k^T\bS)^{\dagger}  \bV_k^T + \bA\bV_k^{\perp} \bV_k^{\perp T} \bS(  \bV_k^T\bS)^{\dagger}  \bV_k^T\|_F 
	\end{align*}
	As seen before, with high probability, $\bV_k^T \bS$ is full rank. Using this fact along with triangle inequality gives us
	\begin{align*}
	&\| \bA - \bA\bS(\bB\bS)^{\dagger}\bB\|_F \\
	&= \| \bA - \bA \bV_k \bV_k^T + \bA\bV_k^{\perp} \bV_k^{\perp T} \bS(  \bV_k^T\bS)^{\dagger}  \bV_k^T\|_F  \\
	&\leq \| \bA - \bA \bV_k \bV_k^T\|_F + \|\bA\bV_k^{\perp} \bV_k^{\perp T} \bS(  \bV_k^T\bS)^{\dagger}  \bV_k^T\|_F   
	\end{align*}
	Define $ \Omega := (\bV_k^T \bS)^{\dagger} - (\bV_k^T \bS)^T$, 
	\begin{align*}
	&\| \bA - \bA\bS(\bB\bS)^{\dagger}\bB\|_F \\
	&= \| \bA\bV_k^{\perp} \bV_k^{\perp T}\|_F + \|\bA\bV_k^{\perp} \bV_k^{\perp T} \bS (\Omega + (  \bV_k^T\bS)^T) \|_F \\ 
	&\leq \| \bA\bV_k^{\perp} \bV_k^{\perp T}\|_F + \|\bA\bV_k^{\perp} \bV_k^{\perp T} \bS\|_F \|\Omega\|_2 + \|\bA\bV_k^{\perp} \bV_k^{\perp T} \bS\bS^T\bV_k \|_F  
	\end{align*}
	By (\ref{eqnS}) and since $\bV^{\perp T}_k\bV_k = 0$, 
	\begin{align*}
	&\| \bA - \bA\bS(\bB\bS)^{\dagger}\bB\|_F\\
	 &\leq \left(1 + \frac{1}{\delta'}\|\Omega\|_2\right)\| \bA\bV_k^{\perp} \bV_k^{\perp T}\|_F + \| \bA \bV^{\perp}_k \bV^{\perp T}_k\bV_k -  \bA \bV^{\perp}_k \bV^{\perp T}_k\bS \bS^T \bV_k  \|_F 
	 \end{align*}
	Using Lemma~\ref{lemma4} and $\|\bV_k \|_F = \sqrt{k}$, 
\begin{align*}
&\| \bA - \bA\bS(\bB\bS)^{\dagger}\bB\|_F \\
&\leq \left(1 + \frac{1}{\delta'}\|\Omega\|_2\right) \| \bA\bV_k^{\perp} \bV_k^{\perp T}\|_F +  \frac{1}{\delta_2 \sqrt{\alpha_R g}} \| \bA \bV^{\perp}_k \bV^{\perp T}_k\|_F \|\bV_k  \|_F \\
& \leq \left(1 + \frac{1}{\delta'}\|\Omega\|_2\right) \| \bA\bV_k^{\perp} \bV_k^{\perp T}\|_F + \frac{\sqrt{ k}}{ \delta_2 \sqrt{\alpha_R g}}  \| \bA\bV_k^{\perp} \bV_k^{\perp T}\|_F \\
& \leq  \left(1 + \frac{1}{\delta'}\|\Omega\|_2 + \frac{\varepsilon_1}{\sqrt{8}\delta_2} \right)  \| \bA \bV^{\perp}_k \bV^{\perp T}_k\|_F
\end{align*}
		where the second inequality follows since  $\| \bA - \bA\bB^{\dagger} \bB\|_F =  \| \bA \bV^{\perp}_k \bV^{\perp T}_k\|_F$ and the last inequality follows since $\frac{\sqrt{ k}}{\sqrt{\alpha_R g}} \leq \frac{ k}{\delta_1\sqrt{\alpha_R g}} \leq \frac{\varepsilon_1}{\sqrt{8}}$.

Finally, using $\bA \bV_k \bV_k^T = \bA \bB^{\dagger} \bB$, we have 
	$$\| \bA - \bA\bS(\bB\bS)^{\dagger}\bB\|_F \leq  \left(1 + \frac{1}{\delta'}\|\Omega\|_2 + \frac{\varepsilon_1}{\sqrt{8}\delta_2} \right) \| \bA - \bA \bB^{\dagger} \bB\|_F$$

	Thus, we can conclude the following with probability $\geq 1 -  (\delta' + \delta_1 + \delta_2) = 1 - \delta$ 
	\begin{align*}
	\| \bA - \bA\bS(\bB\bS)^{\dagger}\bB\|_F & \leq \left(1 + \left(\frac{1}{\sqrt{2}\delta'}+ \frac{1}{2\delta_2} \right) \varepsilon_1 \right) \| \bA - \bA\bB^{\dagger} \bB\|_F  \\
	& \leq \left( 1+  \varepsilon \right) \| \bA - \bA\bB^{\dagger} \bB\|_F
	\end{align*}
	
	by setting  $\delta' = \delta_1 = \delta_2 = \delta/3$ and $\varepsilon = \frac{6\varepsilon_1}{\delta}  $. 
	
	Lemma 2 is used, then $g \geq 36*8 \alpha_R \frac{k^2}{\varepsilon^2\delta^4}$. Finally, note that $\varepsilon_1 \leq \varepsilon < 1 $ by assumption.
	
\end{proof}

\subsection{Proof of Lemma~\ref{lemma4}}
\begin{proof}
Note that,
$$\bE\left[\left( \frac{\bA^{(j_t)} \bB_{(j_t)} }{g p_{j_t}} \right)_{i_1 i_2}\right]  = \sum_{k = 1}^{G} p_k \left( \frac{\bA^{(k)} \bB_{(k)} }{g p_k} \right)_{i_1 i_2} = \frac{1}{g} (\bA\bB)_{i_1 i_2} $$

Since each block is picked independently we have, 
\begin{align*}
 var[(\bC\bR)_{i_1 i_2}] & = var \left[ \sum_{t = 1}^{g} \left( \frac{\bA^{(j_t)} \bB_{(j_t)} }{g p_{j_t}} \right)_{i_1 i_2}\right]\\
 & =  \sum_{t = 1}^{g} var \left[ \left( \frac{\bA^{(j_t)} \bB_{(j_t)} }{g p_{j_t}} \right)_{i_1 i_2}\right]\\
 & =  \sum_{t = 1}^{g}\left( \bE\left[  \left( \frac{\bA^{(j_t)} \bB_{(j_t)} }{g p_{j_t}} \right)_{i_1 i_2}^2\right] - \bE\left[  \left( \frac{\bA^{(j_t)} \bB_{(j_t)} }{g p_{j_t}} \right)_{i_1 i_2}\right]^2\right)\\
 & =g \left( \sum_{k = 1}^{G} p_k \left( \frac{\bA^{(k)} \bB_{(k)} }{g p_k} \right)_{i_1 i_2}^2 - \frac{(\bA\bB)_{i_1 i_2}^2}{g^2} \right)\\
  & =\frac{1}{g} \left( \sum_{k = 1}^{G}  \frac{\left( \bA^{(k)} \bB_{(k)} \right)_{i_1 i_2}^2 }{p_k}  - (\bA\bB)_{i_1 i_2}^2 \right)
 \end{align*}

\begin{align*}
\E [ \| \bA\bB - \bC\bR \|_F^2 ] & = \sum_{i_1 = 1}^{m}\sum_{i_2 = 1}^{p}  var[(\bC\bR)_{i_1 i_2}] \\
& = \sum_{i_1 = 1}^{m}\sum_{i_2 = 1}^{p} \frac{1}{g} \left( \sum_{k = 1}^{G}  \frac{\left( \bA^{(k)} \bB_{(k)} \right)_{i_1 i_2}^2 }{p_k}  - (\bA\bB)_{i_1 i_2}^2 \right)\\
& =\left( \sum_{k = 1}^{G} \frac{1}{g p_k}  \sum_{i_1 = 1}^{m}\sum_{i_2 = 1}^{p}   \left( \bA^{(k)} \bB_{(k)} \right)_{i_1 i_2}^2 \right)  - \frac{\|\bA\bB\|_F^2}{g}\\
& = \sum_{k = 1}^{G} \frac{\| \bA^{(k)} \bB_{(k)}\|_F^2}{g p_k}      - \frac{\|\bA\bB\|_F^2}{g}\\
& \leq \sum_{k = 1}^{G} \frac{\| \bA^{(k)} \|_2^2 \|\bB_{(k)}\|_F^2}{g p_k}     \\
& \leq \sum_{k = 1}^{G}\left( \frac{\| \bA^{(k)}\|_2 }{\| \bA^{(k)}\|_F} \right)^2\frac{\| \bA^{(k)} \|_F^2 \|\bB_{(k)}\|_F^2}{g p_k}     \\
& \leq  \frac{1}{\beta \alpha_A g}  \| \bA \|_F^2  \|\bB\|_F^2   \\
\end{align*}

where $\alpha_A = \min_k\left( \frac{\| \bA^{(k)}\|_F }{\| \bA^{(k)}\|_2} \right)^2$. Also, note $1 \leq \alpha_A \leq s$.\\

By Jensen's inequality,
\begin{align*}
\E [ \| \bA\bB - \bC\bR \|_F ]^2 & \leq \E [ \| \bA\bB - \bC\bR \|_F^2 ]\\
& \leq  \frac{1}{\beta \alpha_A g}   \| \bA \|_F^2  \|\bB\|_F^2 \\
\end{align*}
And by Markov's inequality, with probability $\geq 1-\delta$, we have
\begin{align*}
\| \bA\bB - \bC\bR \|_F \leq \frac{1}{\delta}\E [ \| \bA\bB - \bC\bR \|_F ]  \leq  \frac{1}{\delta \sqrt{\beta \alpha_A g} }   \| \bA \|_F  \|\bB\|_F 
\end{align*}

\end{proof}

 \subsection{Proof of Corollary 1}
 
 Here we state and prove the corollary mentioned in Section 3.2 of the paper. If it is possible to compute the SVD of the entire matrix, then the rows can be sampled using row leverage scores, and the incoherence assumption can be dropped. The relative error guarantee for the full SVD Block CUR approximation is stated below.
 
 \begin{corollary}\label{Thm1}
Given $\bA \in \R^{m \times n}$, let $r = O(\frac{k^2}{\varepsilon^{2}} \ln(\frac{1}{\delta}))$ and $g = O(\frac{ r^2}{\alpha_R \varepsilon^{2}} \ln(\frac{1}{\delta}) )$. There exist randomized algorithms such that, if $r$ rows and $g$ column blocks are chosen to construct $\bR$ and $\bC$, respectively, then with probability $\geq 1 - \delta$, the following holds:
\begin{equation*} \| \bA - \bC \bU \bR\|_F \leq (1 + \varepsilon ) \| \bA -  \bA_k \|_F, \end{equation*}
  where $\varepsilon, \delta \in (0,1)$, and $\bU = \bW^{\dagger}$ is the pseudoinverse of the scaled intersection of $\bC$ and $\bR$.   \end{corollary}
 
 First, note $\bU = (\bR\bS)^{\dagger}$ and $\bC = \bA\bS$. 
 \begin{align*}
 \|\bA - \bC\bU\bR\|_F &= \|\bA - \bA\bS (\bR\bS)^{\dagger}\bR\|_F. 
\end{align*}

Similar to the proof of Lemma 1, we can use Lemma 1.1 with probability at least 0.85 we have
 $$ \| \bA - \bA\bS(\bR\bS)^{\dagger}\bR\|_F \leq (1 + \varepsilon ) \| \bA - \bA\bR^{\dagger}\bR\|_F. $$

 Recall that $\bA \in \R^{m \times n}; \varepsilon \in (0,1)$; and that the rows $\bR$ are picked from $\bA$ with the following probability distribution:
 $$ p_i = \frac{ \|e_i^T \bU_{A,k} \|_2^2}{k}, \ \forall i \in [m]. $$

We bound $\|\bA - \bA\bR^{\dagger}\bR \|_F$ using Theorem 1 in \cite{drineas08} we get with probability at least 0.85,
$\|\bA - \bA\bR^{\dagger}\bR \|_F \leq (1 + \varepsilon) \| \bA - \bA_k\|_F$. Finally, we get with probability 0.7
\begin{align*}
 \|\bA - \bC\bU\bR\|_F & \leq (1 + \varepsilon' )^2 \| \bA - \bA_k \|_F&\\
& \leq (1 + \varepsilon'' ) \| \bA - \bA_k \|_F&(\textnormal{letting } \varepsilon'' = 3\varepsilon')
\end{align*}
This completes the proof of Corollary 1.

\end{document}